%% file: iclr2025_conference.tex
\pdfoutput=1
\documentclass{article} 
\usepackage{times}
\usepackage{ICLR_2025_Template/iclr2025_conference}
\input{ICLR_2025_Template/math_commands.tex}

\usepackage[utf8]{inputenc} 
\usepackage[T1]{fontenc}    
\usepackage{hyperref}       
\usepackage{url}            
\usepackage{booktabs}       
\usepackage{amsfonts}       
\usepackage{nicefrac}       
\usepackage{microtype}      
\usepackage{xcolor}         
\usepackage{amsmath,amsfonts,amsthm,bm}
\usepackage{multirow}
\usepackage{graphicx}
\usepackage{colortbl}
\usepackage{wrapfig}
\usepackage{CJKutf8}
\usepackage{thmtools}
\usepackage{thm-restate}
\usepackage{subcaption}
\usepackage{tablefootnote}
\usepackage{algorithm}
\usepackage{algorithmic}
\theoremstyle{plain}
\newtheorem{theorem}{Theorem}
\newtheorem{lemma}[theorem]{Lemma}
\newtheorem{definition}{Definition}

\usepackage{xspace}

\usepackage[capitalize]{cleveref}
\crefname{section}{Sec.}{Secs.}
\Crefname{section}{Section}{Sections}
\Crefname{table}{Table}{Tables}
\crefname{table}{Tab.}{Tabs.}
\input{mathcmd}

\newcommand{\modelname}{CrysBFN\xspace}

\iclrfinalcopy
\title{A Periodic Bayesian Flow for Material Generation}

\author{Hanlin Wu$^{1*}$ \quad Yuxuan Song$^{1,2}$\thanks{Equal Contribution. Correspondence to Hao Zhou (zhouhao@air.tsinghua.edu.cn) and Yawen Ouyang (ouyangyawen@air.tsinghua.edu.cn). } \quad Jingjing Gong$^1$ \quad Ziyao Cao$^{1,2}$ \quad Yawen Ouyang$^1$ 
 \\\textbf{Jianbing Zhang}$^3$ \quad \textbf{Hao Zhou}$^{1}$ \quad \textbf{Wei-Ying Ma}$^{1}$ \quad \textbf{Jingjing Liu}$^{1}$\\
$^1$ Institute of AI Industry Research (AIR), Tsinghua University\\
$^2$ Dept. of Comp. Sci. \& Tech., Tsinghua University\\
$^3$ School of Artifcial Intelligence, Nanjing University \\
\texttt{wuhl24@mails.tsinghua.edu.cn},\\\texttt{\{songyuxuan,zhouhao,ouyangyawen,jjliu\}@air.tsinghua.edu.cn} \\
}

\begin{document}

\maketitle

\input{abstract}

\input{intro}
\begin{figure}[t]
    \vskip -0.2in
    \centering
    \includegraphics[width=\textwidth]{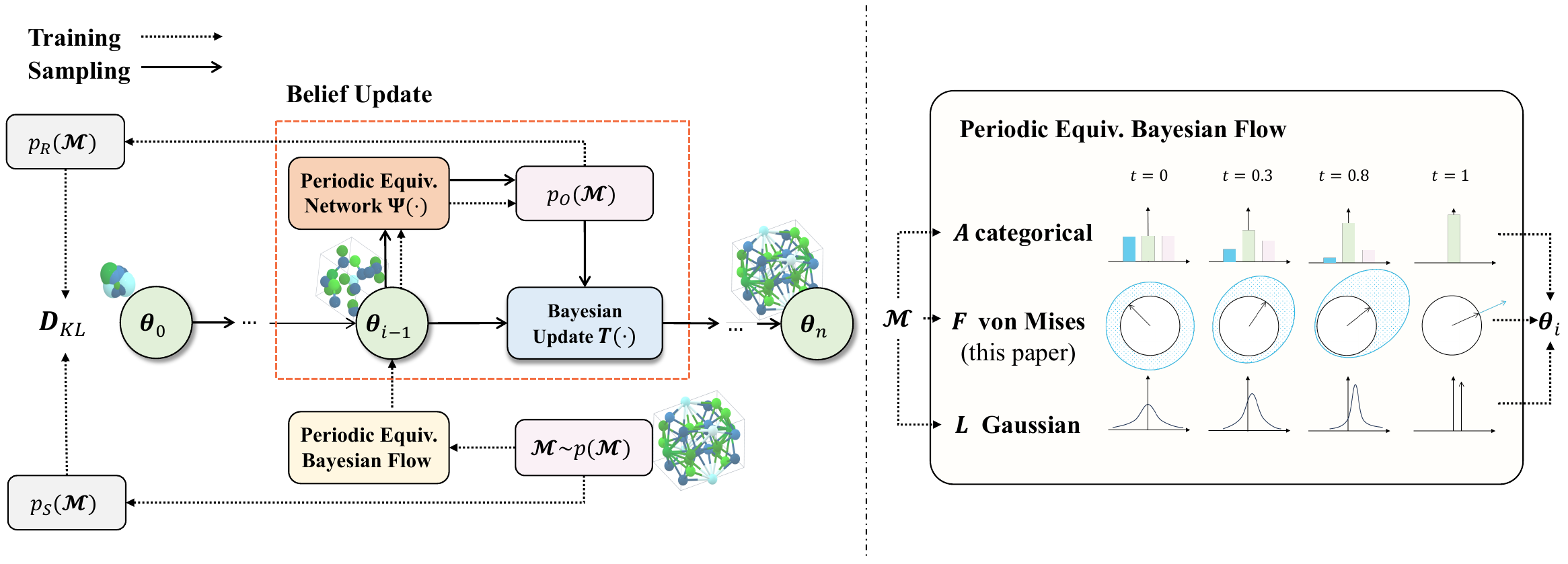}
    \caption{\textbf{Framework of CrysBFN}. Left: overview of training and sampling process. At training time, the network receives $\btheta_{i-1}$ from Bayesian flow based on data distribution, and tries to improve the belief $\btheta_{i-1}$ over the groundtruth $\calM$ by outputting an estimated distribution $p_O$ and minimizing the gap between estimation and groundtruth. At sampling time with the trained network, the uninformative prior $\btheta_0$ is gradually improved by belief updates until $\btheta_n$ with high fidelity. Right: illustration of the  periodic equivariant Bayesian flow.}
    \label{fig:framework}
    \vskip -0.2in
\end{figure}

\section{Related Work}

Modeling and generating stable materials with data-driven approaches has been applied to discovering new functional materials \citep{peng2022human}. One line of approaches indirectly models crystal space by transforming crystals into human-designed representations \citep{REN2021,hoffmann2019data,noh2019inverse, court20203,yang2023scalable}, though the encoding and decoding process often leads to physical geometry loss 
In contrast, another line of research directly models crystals in the sample space, drawing inspiration from the success of Diffusion models~\citep{ho2020denoising,song2020score,song2019generative}. For instance, CD-VAE \citep{xie2021crystal} and SyMat \citep{luo2024towards} employ score-matching \citep{song2019generative} to learn scores for generating stable materials, while their modeled distribution lacks geometric invariance~\citep{zhang2023artificial}. DiffCSP~\citep{jiao2023crystal} addresses this by transforming Cartesian atom coordinates into fractional coordinates, introducing the periodic E(3) equivariance of crystals, and designing an equivariant diffusion crystal generation model based on periodic diffusion \citep{jing2022torsional}. More recently, \citet{flowmm} applies Riemannian Flow Matching \citep{riemannianfm} to crystal generation tasks offering improved sampling efficiency, while at the expense of quality \footnote{For a more detailed discussion of related work, please refer to \cref{appd:detailed_related_works}.}. However, we argue that these methods struggle to balance sampling quality and efficiency due to insufficient guidance during each transition from the noise prior to the data distribution. This issue is particularly pronounced for crystals, where thermodynamically stable materials constitute only a small fraction of the search space~\citep{flowmm}. For instance, early generation states $x_{t-1}$ with low confidence should be retained less than later states to achieve the next state $x_{t}$.

In this work, we propose to use BFN~\citep{bfn} to model crystals in a principally different way. BFN provides a framework to accurately update each generation state according to its entropy/confidence, the effectiveness of which has been proved in \citet{song2023unified}. 
However, there are no established explorations on the challenging topic of non-Euclidean BFN which is essential to many real-world applications \cite{jing2022torsional, AlphaFold2021}. 
To address the above issues, in this paper, we build a non-Euclidean Bayesian flow from scratch, identifying and tackling the non-additive accuracy via introducing a novel entropy conditioning mechanism.
 

\input{preliminaries}
\input{method}

\input{experiment}

\section{Conclusion}\label{sec:discussion}
In this paper, we present the first periodic Bayesian flow modeling on the hyper-torus, addressing an unprecedented theoretical issue related to non-additive accuracy. Specifically, we introduce a novel entropy conditioning mechanism, theoretical reformulations of BFN, a fast sampling algorithm, and a numerical method for determining the accuracy schedule. Leveraging the proposed periodic Bayesian flow, we implement the first periodic E(3) equivariant Bayesian flow networks for crystal generation. Our approach achieves state-of-the-art performance in crystal generation, with efficiency improved by two orders of magnitude. Additionally, our methodology can be adapted to a wide range of data types and tasks involving hyper-torus data.

\newpage

\section*{Ethics Statement}
We confirm that our work complies with the ICLR Code of Ethics, and we have carefully considered potential ethical concerns related to the development and use of our proposed method, \modelname, for crystal generation. Our method is designed for general crystal generative modeling tasks and does not involve sensitive data or tasks. We strongly encourage users to ensure compliance with relevant privacy regulations and critically assess the model's outputs. We are confirmed that there is no conflict of interest, financial sources or other factors, that could influence the development or presentation of this work. 

With these considerations, we do not anticipate any violations of the ICLR Code of Ethics in the development or use of this model. And we stress once again that \modelname should not be used for malicious purposes, such as creating harmful structures.


\section*{Reproducibility Statement}
To ensure reproducibility, we give a detailed derivation of the periodic Bayesian flow in \cref{appd:bfn_cir} and the proof of the propositions in \cref{appd:geometric_invar}. All datasets and performance evaluation method used in our experiments are publicly available and clearly specified or cited in \cref{sec:exp}. We provide our implementation details including training and sampling procedure, hyper-parameters, used computational resources and anonymous code repository link in \cref{appd:imple_details}.

\subsubsection*{Acknowledgments}
This work is supported by the National Science and Technology Major Project (2022ZD0117502), the Natural Science Foundation of China (Grant No. 62376133) and sponsored by Beijing Nova Program (20240484682).


\bibliographystyle{iclr2025_conference}

\input{ICLR_2025_Template/bibs}
\input{appendix}

\end{document}

%% file: ICLR_2025_Template/math_commands.tex
\usepackage{amsmath,amsfonts,bm}

\def\eqref#1{equation~\ref{#1}}

\def\1{\bm{1}}

\def\rvx{{\mathbf{x}}}

\def\vone{{\bm{1}}}
\def\vmu{{\bm{\mu}}}
\def\vtheta{{\bm{\theta}}}

\def\vc{{\bm{c}}}

\def\vh{{\bm{h}}}

\def\vm{{\bm{m}}}

\def\vt{{\bm{t}}}

\def\vx{{\bm{x}}}
\def\vy{{\bm{y}}}
\def\vz{{\bm{z}}}

\def\mA{{\bm{A}}}

\def\mF{{\bm{F}}}

\def\mI{{\bm{I}}}

\def\mL{{\bm{L}}}

\def\mP{{\bm{P}}}
\def\mQ{{\bm{Q}}}

\DeclareMathAlphabet{\mathsfit}{\encodingdefault}{\sfdefault}{m}{sl}
\SetMathAlphabet{\mathsfit}{bold}{\encodingdefault}{\sfdefault}{bx}{n}

\newcommand{\E}{\mathbb{E}}

\newcommand{\R}{\mathbb{R}}



%% file: mathcmd.tex
\newcommand{\defeq}{%
    \stackrel{\text{def}}{=} 
}
\DeclareRobustCommand{\btheta}{\bm{\theta}}

\newcommand{\bmuL}{\bm{\mu}^L}
\newcommand{\brhoL}{\bm{\rho}^L}
\newcommand{\bthetaA}{\bm{\theta}^A}
\newcommand{\bthetaF}{\bm{\theta}^F}
\newcommand{\bthetaL}{\bm{\theta}^L}
\newcommand{\bthetaM}{\bm{\theta}^\mathcal{M}}
\newcommand{\bmF}{\bm{m}^F}
\newcommand{\bcF}{\bm{c}^F}

\newcommand{\dotx}{\dot{\bm{x}}}
\newcommand{\doty}{\dot{\bm{y}}}
\newcommand{\dotm}{\dot{\bm{m}}}

\def\vone{{\bm{1}}}
\def\vmu{{\bm{\mu}}}
\def\vtheta{{\bm{\theta}}}

\def\vc{{\bm{c}}}

\def\vh{{\bm{h}}}

\def\vm{{\bm{m}}}

\def\vt{{\bm{t}}}

\def\vx{{\bm{x}}}
\def\vy{{\bm{y}}}
\def\vz{{\bm{z}}}
\def\vL{{\bm{L}}}
\def\vA{{\bm{A}}}
\def\vR{{\bm{R}}}
\def\vQ{{\bm{Q}}}
\def\vI{{\bm{I}}}

\def\vF{{\bm{F}}}

\def\vrho{{\bm{\rho}}}

\def\calM{\mathcal{M}}
\def\calS{\mathcal{S}}

\def\calL{\mathcal{L}}

\DeclareMathOperator{\x}{\mathbf{x}}

\DeclareMathOperator{\y}{\mathbf{y}}

\renewcommand{\k}{\mathbf{k}}

\newcommand{\kl}[2]{D_{KL}\left(#1 \parallel #2\right)}
\newcommand{\N}[2]{\mathcal{N}\left(#1 , #2\right)}

\newcommand{\tidx}[2]{#1_{#2}}
\newcommand{\didx}[2]{#1^{(#2)}}
\renewcommand{\vec}[1]{\boldsymbol{#1}}
\newcommand{\pars}{\theta}
\newcommand{\parsn}{\vec{\pars}}

\newcommand{\parsnt}[1]{\tidx{\parsn}{#1}}
\newcommand{\alphat}[1]{\tidx{\alpha}{#1}}

\newcommand{\yd}{y}
\newcommand{\ydd}[1]{\didx{\yd}{#1}}

\newcommand{\oh}[2]{\mathbf{e}_{#1}}
\newcommand{\ds}[1]{\{1,#1\}}

\newcommand{\ui}[1]{U\ds{#1}}

\newcommand{\net}{\Psi}
\newcommand{\sender}[2]{p_{_S}\left(#1 \mid #2\right)}
\newcommand{\out}{p_{_O}}

\newcommand{\rec}{p_{_R}}

\newcommand{\flow}{p_{_F}}
\newcommand{\update}{p_{_U}}

\newcommand{\thetaM}{\btheta^\mathcal{M}}
\newcommand{\senderacc}{\alpha_1,\alpha_2,\dots,\alpha_i}
\newcommand{\predL}[1]{\hat{\net}_L(\thetaM_{#1}, t_{#1})}

\newcommand{\predF}[1]{\hat{\net}_F(\thetaM_{#1}, t_{#1})}
\newcommand{\predA}[1]{\hat{\net}_A(\thetaM_{#1}, t_{#1})}
\newcommand{\predM}[1]{\hat{\net}_M(\thetaM_{#1}, t_{#1})}

\def\mA{{\bm{A}}}

\def\mF{{\bm{F}}}

\def\mI{{\bm{I}}}

\def\mL{{\bm{L}}}

\def\mP{{\bm{P}}}
\def\mQ{{\bm{Q}}}

%% file: abstract.tex
\begin{abstract}
Generative modeling of crystal data distribution is an important yet challenging task due to the unique periodic physical symmetry of crystals. Diffusion-based methods have shown early promise in modeling crystal distribution. 
More recently, Bayesian Flow Networks were introduced to aggregate noisy latent variables, resulting in a variance-reduced parameter space that has been shown to be advantageous for modeling Euclidean data distributions with structural constraints \citep{song2023unified}.
Inspired by this, we seek to unlock its potential for modeling variables located in non-Euclidean manifolds \textit{e.g.} those within crystal structures, by overcoming challenging theoretical issues. We introduce CrysBFN, a novel crystal generation method by proposing a periodic Bayesian flow, which essentially differs from the original Gaussian-based BFN by exhibiting non-monotonic entropy dynamics. To successfully realize the concept of periodic Bayesian flow, CrysBFN integrates a new entropy conditioning mechanism and empirically demonstrates its significance compared to time-conditioning. Extensive experiments over both crystal ab initio generation and crystal structure prediction tasks demonstrate the superiority of CrysBFN, which consistently achieves new state-of-the-art on all benchmarks. Surprisingly, we found that CrysBFN enjoys a significant improvement in sampling efficiency, \emph{e.g.}, $\sim$ 100$\times$ speedup (10 \textit{v.s.} 2000 steps network forwards) compared with previous diffusion-based methods on MP-20 dataset. Code is available at \url{https://github.com/wu-han-lin/CrysBFN}.

\end{abstract}

%% file: intro.tex
\section{Introduction}


Deep generative models, with their strong ability to approximate data distribution with complex geometries, have recently emerged as a promising approach to de novo drug design~\citep{hoogeboom2022equivariant}, protein engineering~\citep{shi2022protein}, and material science~\citep{material_dicovery}.  
To discover new functional materials~\citep{wang2023scientific,peng2022human}, there has been an active line of research on crystal generative modeling~\citep{REN2021,hoffmann2019data,noh2019inverse, court20203,yang2023scalable,nouira2018crystalgan}. Recent diffusion-based models learns through an iterative reverse process with multi-level noise perturbation, and has been demonstrated as a powerful tool for capturing complex geometries of crystals. Studies show that these models can generate crystal samples with realistic structures that well satisfy physical constraint~\citep{xie2021crystal,jiao2023crystal,jiao2024space,equicsp}.


Despite promising results, significant challenges persist. The search space for crystal structures grows exponentially with the number of atoms, while thermodynamically stable materials represent only a small fraction~\citep{flowmm}. This presents challenges in the multi-step generation process, the variance of which might cause structures to deviate from stable distributions. Moreover, the current widely-adapted diffusion-based approaches~\citep{jiao2023crystal,jiao2024space} for crystal structure modeling tend to learn the score function of wrapped normal distribution for periodic variables, where the approximation of a sum of infinite terms is needed which could bring in extra bias. Recently, BFN~\citep{bfn} has been successfully applied to the geometry generative modeling of molecules~\citep{song2023unified}, a scenario that shares the above-mentioned challenges similar to crystal generation, by modeling in a much lower variance parameter space. However, the periodic geometry of crystals differs from that of small molecules and raises significant challenges.

To tackle these challenges, this paper aim to break the barrier of extending the paradigm of BFN into those variables located non-Euclidean space, \textit{e.g.}, atom fractional coordinates in crystal structure. We introduce the first non-Euclidean Bayesian flow over the periodic space, \textit{i.e.} the hyper-torus. To successfully implement such concept, we introduce a generalized training paradigm based on simulation of the Bayesian flow and further propose a non-auto-regressive equivalent formulation of Bayesian flow distribution that guarantees computational efficiency. By integrating all these innovations, we introduce \modelname, the first periodic E(3) equivariant Bayesian flow network designed for crystal generation. Extensive experiments demonstrate the significant superiority of \modelname over current methods in both sampling quality and efficiency.

Our contributions can be summarized as follows:
\begin{itemize}
    \item We present the first periodic Bayesian flow in non-Euclidean space (hyper-torus) with a novel training paradigm and entropy conditioning mechanism tackling the unprecedented and pivotal non-additive accuracy theoretical challenge.
    \item We introduce the first periodic-E(3) equivariant Bayesian flow networks for crystal generation tasks with appealing theoretical guarantees.
    \item Extensive experiments demonstrate that \modelname consistently outperforms previous methods on both ab initio crystal generation (99.1\% COV-P on Carbon-24) and crystal structure prediction tasks (64.35$\%$ match rate on MP-20). Efficiency experiments on MP-20 prove that \modelname enjoys a \textbf{$\sim$ 100$\times$ sampling efficiency} with performance on par with previous Diffusion-based methods.
\end{itemize}

%% file: preliminaries.tex
\section{Preliminaries}\label{sec:preliminaries}

\textbf{Crystal Representation and Related Manifold} Crystals can be represented as a structure composed of infinite, periodic and repeating unit cells defined by a triplet $\mathcal{M}=(\boldsymbol{A},\boldsymbol{F},\boldsymbol{L})$. Denote the number of atoms in the unit cell as $N$, $\boldsymbol{A}=(\vec{a}_1,\vec{a}_2,\dots,\vec{a}_N)\in \mathcal{S}^{K\times N}$ is the representation of atom types with the length of vocabulary $K$, and every such one-hot discrete variable locates in the simplex $\mathcal{S}^K$ represented by:
\begin{equation}
\mathcal{S}^K\defeq\{\boldsymbol{s}\in \mathbb{R}^K|\sum_{i=1}^{K}s_i=1,s_i\geq0, i=1,\dots,K\}
\end{equation} 
which requires the designed generative path for $\boldsymbol{A}$ should be well defined on the simplex.
Following \citet{jiao2023crystal}, $\boldsymbol{F}=[\boldsymbol{f}_1,\boldsymbol{f}_2,\dots,\boldsymbol{f}_N]\in [0,1)^{3\times N}$ is the fractional coordinates of atoms located in quotient space $\R^{3\times N} / \mathbb{Z}^{3\times N}$ equivalent to the hypertorus $\mathbb{T}^{3\times N}$~\citep{jing2022torsional}. The hypertorus $\mathbb{T}^{3\times N}$ can be represented as the Cartesian product of  $3\times N$ toruses $\mathbb{T}^1$:
\begin{equation}
    \mathbb{T}^1\defeq \{\vz\in\mathbb{R}^2:||\vz||=1\}
\end{equation}
And $\boldsymbol{L}=[\boldsymbol{l}_1,\boldsymbol{l}_2,\boldsymbol{l}_3]\in \mathbb{R}^{3\times 3}$ denotes the lattice matrix, every column vector of which is the periodic basic vector of the crystal. We can get the Cartesian coordinates representation of unit cell's atom coordinates $\bm{X}$ by $\boldsymbol{X}=\boldsymbol{L}\boldsymbol{F}\in \mathbb{R}^{3\times N}$. The ideal infinite periodic crystal structure of $\mathcal{M}$ can be represented by $\{(a_i',\boldsymbol{x}_i')|a_i'=a_i,\boldsymbol{x}_i'=\boldsymbol{x}_i+\boldsymbol{L}\boldsymbol{k},\forall\boldsymbol{k}\in \mathbb{Z}^{3\times 1}\}$. Based on the above notations, the symmetry of crystal geometry is defined as \emph{periodic E(3) invariance}\footnote{Actually some crystal structures are chiral \citep{flack2003chiral} and their structures are SE(3) invariant. In this paper, we follow \citet{jiao2023crystal} to build an E(3) invariant framework for fair comparison, while the SE(3) equivariance can be simply achieved by transforming an E(3) invariant network into an SE(3) equivariant network by excluding invariance to reflections in the Markov chain. } \citep{jiao2023crystal}, including periodic translational invariance of $\vF$ and rotational invariance of $\vL$ (details in \cref{appd:geometric_invar}) . 

\textbf{Bayesian Flow Networks}
Different from the well-established SDE-based approaches, \emph{e.g.} Diffusion Models~\citep{ddpm,sde,song2020score} and ODE-based approaches \emph{e.g.} Flow Matching~\citep{flowmatching}, Bayesian Flow Networks define a generative process driven by consecutive Bayesian updates on noised samples from the uninformative prior distribution $\vtheta_{0}$ to posteriors $\vtheta_i$ with higher confidence and more information.

To define the consecutive Bayesian update process, Bayesian Flow Networks contain a process to add noise to the clean data samples which is an analogy to the forward process in Diffusion Models. In BFN, such a process is explicitly defined by the so-called \emph{sender distribution} $p_S(\mathbf{y}|\mathbf{x};\alpha)$ where $\alpha$ is the parameter of sender distribution which corresponds to noise level, \emph{e.g.} variance for Gaussian-formed $p_S$. 

 BFN aims to create a procedure that gradually acquires information from the ground truth data \(\rvx\) to provide a training signal. To this end, the framework will sample a series of noisy samples $\bold{y}_1, \bold{y}_2, ..., \bold{y}_n$  independently from the \emph{sender distribution} $p_S$ with various accuracy levels $\alpha_1, \alpha_2, ..., \alpha_n$. Based on samples, the framework would simulate an auto-regressive update of $\vtheta$ based on the \emph{Bayesian update function} to reflect an information gathering process from prior to data as:
\begin{align}
    \label{eq:bayesian_update}
    \vtheta_i = h(\boldsymbol{\theta}_{i-1},\bold{y}_{i},\alpha_{i})
\end{align}

The neural network $\Psi$ is trained in a teacher-forcing fashion: approximate the sender distribution $p_S(\cdot| \bold{x};\alpha_i)$ generating $\bold{y}_i$ according to $\btheta_{i-1}$, which has integrated the information of $(\bold{y}_1, \cdots, \bold{y}_{i-1})$ through \cref{eq:bayesian_update}. The network $\Psi$ hence takes  $\vtheta$ obtained by Bayesian update as input, and the distribution implied by $p_I(\cdot|\vtheta)$ is termed \emph{input distribution}. The network output $\Psi(\vtheta)$ is interpreted as the parameter of an updated distribution $p_O(\cdot|\Psi(\vtheta))$ over sample space, referred to as \emph{output distribution}. Combining the network output $\Psi(\vtheta)$ and the sender distribution, we obtain the approximation to $p_S\left(\bold{y}_i \mid \bold{x} ; \alpha_i\right)$ as:
\begin{equation}
\label{eq:abs_rece}
    p_R(\bold{y}_i\mid \vtheta_{i-1},\Psi,\alpha_i)= \mathbb{E}_{p_O(\bold{x}' |\Psi(\boldsymbol{\vtheta}_{i-1}))}p_S\left(\bold{y}_i \mid \bold{x}' ; \alpha_i\right)
\end{equation}
Such distribution is named as \emph{receiver distribution} $p_R$. Note that $\vtheta_{i-1}$ can be seen as a deterministic function mapping, given $(\bold{y}_1,\cdots,\bold{y}_{i-1})$  as $\vtheta_{i-1} = f(\bold{y}_1,\cdots,\bold{y}_{i-1})$ based on the deterministic update function in \cref{eq:bayesian_update}. By combining the objectives of different timesteps and taking expectation over trajectories, we obtain the training objective of BFN as:
\begin{align}\label{eq:loss_n}
    \mathcal{L}(\Psi) = \mathbb{E}_{\x \sim p_{\text{data}}}\mathbb{E}_{\Pi_{i=1}^n p_S(\bold{y}_i|\bold{x},\alpha_i)}  D_{K L} (p_S(\bold{y}_i|\bold{x},\alpha_i)||p_R(\bold{y}_i\mid \vtheta_{i-1},\Psi,\alpha_i))
\end{align}
Here $p_{\text{data}}$ is the empirical distribution. Detailed derivation and discussion are provided in \cref{appd:bfn_cir}. Additionally, we provide toy examples with minimal components illustrating how BFNs work in our code repository.

%% file: method.tex
\section{Method}\label{sec:method}
\begin{figure}
    \centering
    \includegraphics[width=0.85\textwidth]{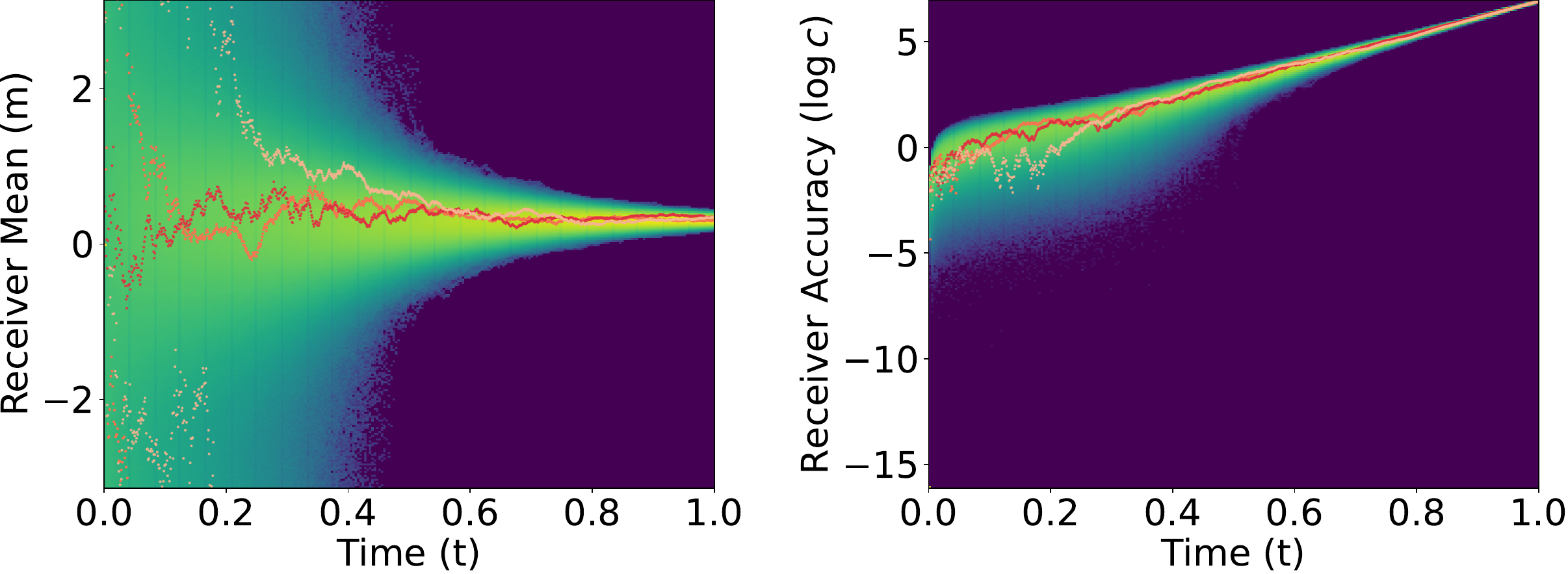}
    \caption{\textbf{Visualization of the proposed periodic Bayesian flow with mean parameter $\mu$ and accumulated accuracy parameter $c$ which corresponds to the entropy/uncertainty}. For $x = 0.3, \beta(1) = 1000$ and $\alpha_i$ defined in \cref{appd:bfn_cir}, this figure plots three colored stochastic parameter trajectories for receiver mean parameter $m$ and accumulated accuracy parameter $c$, superimposed on a log-scale heatmap of the Bayesian flow distribution $p_F(m|x,\senderacc)$ and $p_F(c|x,\senderacc)$. Note the \emph{non-monotonicity} and \emph{non-additive} property of $c$ which could inform the network the entropy of the mean parameter $m$ as a condition and the \emph{periodicity} of $m$. 
    }
    \label{fig:vmbf_vis}
    \vskip -0.1in
\end{figure}

In this section, we explain the detailed design of CrysBFN tackling theoretical and practical challenges. First, we describe how to derive our new formulation of Bayesian Flow Networks over hyper-torus $\mathbb{T}^{D}$ from scratch. Next, we illustrate the two key differences between \modelname and the original form of BFN: $1)$ a meticulously designed novel base distribution with different Bayesian update rules; and $2)$ different properties over the accuracy scheduling resulted from the periodicity and the new Bayesian update rules. Then, we present in detail the overall framework of \modelname over each manifold of the crystal space (\textit{i.e.} fractional coordinates, lattice vectors, atom types) respecting \textit{periodic E(3) invariance}. 


\subsection{Periodic Bayesian Flow on Hyper-torus \texorpdfstring{$\mathbb{T}^{D}$}{}} 
For generative modeling of fractional coordinates in crystal, we first construct a periodic Bayesian flow on \texorpdfstring{$\mathbb{T}^{D}$}{} by designing every component of the totally new Bayesian update process which we demonstrate to be distinct from the original Bayesian flow (please see \cref{fig:non_add}). 
 
 The fractional atom coordinate system \citep{jiao2023crystal} inherently distributes over a hyper-torus support $\mathbb{T}^{3\times N}$. Hence, the normal distribution support on $\R$ used in the original \citep{bfn} is not suitable for this scenario. 

To tackle this problem, the circular distribution~\citep{mardia2009directional} over the finite interval $[-\pi,\pi)$ is a natural choice as the base distribution for deriving the BFN on $\mathbb{T}^D$. 
Specifically, circular distributions enjoy desirable periodic properties: $1)$ the integration over any interval length of $2\pi$ equals 1; $2)$ the probability distribution function is periodic with period $2\pi$.  Sharing the same intrinsic with fractional coordinates, such periodic property of circular distribution makes it suitable for the instantiation of BFN's input distribution, in parameterizing the belief towards ground truth $\x$ on $\mathbb{T}^D$. 

\textbf{von Mises Distribution and its Bayesian Update} We choose von Mises distribution \citep{mardia2009directional} from various circular distributions as the form of input distribution, based on the appealing conjugacy property required in the derivation of the BFN framework.
That is, the posterior of a von Mises distribution parameterized likelihood is still in the family of von Mises distributions. The probability density function of von Mises distribution with mean direction parameter $m$ and concentration parameter $c$ (describing the entropy/uncertainty of $m$) is defined as: 
\begin{equation}
f(x|m,c)=vM(x|m,c)=\frac{\exp(c\cos(x-m))}{2\pi I_0(c)}
\end{equation}
where $I_0(c)$ is zeroth order modified Bessel function of the first kind as the normalizing constant. Given the last univariate belief parameterized by von Mises distribution with parameter $\theta_{i-1}=\{m_{i-1},\ c_{i-1}\}$ and the sample $y$ from sender distribution with unknown data sample $x$ and known accuracy $\alpha$ describing the entropy/uncertainty of $y$,  Bayesian update for the receiver is deducted as:
\begin{equation}
 h(\{m_{i-1},c_{i-1}\},y,\alpha)=\{m_i,c_i \}, \text{where}
\end{equation}
\begin{equation}\label{eq:h_m}
m_i=\text{atan2}(\alpha\sin y+c_{i-1}\sin m_{i-1}, {\alpha\cos y+c_{i-1}\cos m_{i-1}})
\end{equation}
\begin{equation}\label{eq:h_c}
c_i =\sqrt{\alpha^2+c_{i-1}^2+2\alpha c_{i-1}\cos(y-m_{i-1})}
\end{equation}
The proof of the above equations can be found in \cref{apdx:bayesian_update_function}. The atan2 function refers to  2-argument arctangent. Independently conducting  Bayesian update for each dimension, we can obtain the Bayesian update distribution by marginalizing $\y$:
\begin{equation}
p_U(\vtheta'|\vtheta,\bold{x};\alpha)=\mathbb{E}_{p_S(\bold{y}|\bold{x};\alpha)}\delta(\vtheta'-h(\vtheta,\bold{y},\alpha))=\mathbb{E}_{vM(\bold{y}|\bold{x},\alpha)}\delta(\vtheta'-h(\vtheta,\bold{y},\alpha))
\end{equation} 
\begin{figure}
    \centering
    \vskip -0.15in
    \includegraphics[width=0.95\linewidth]{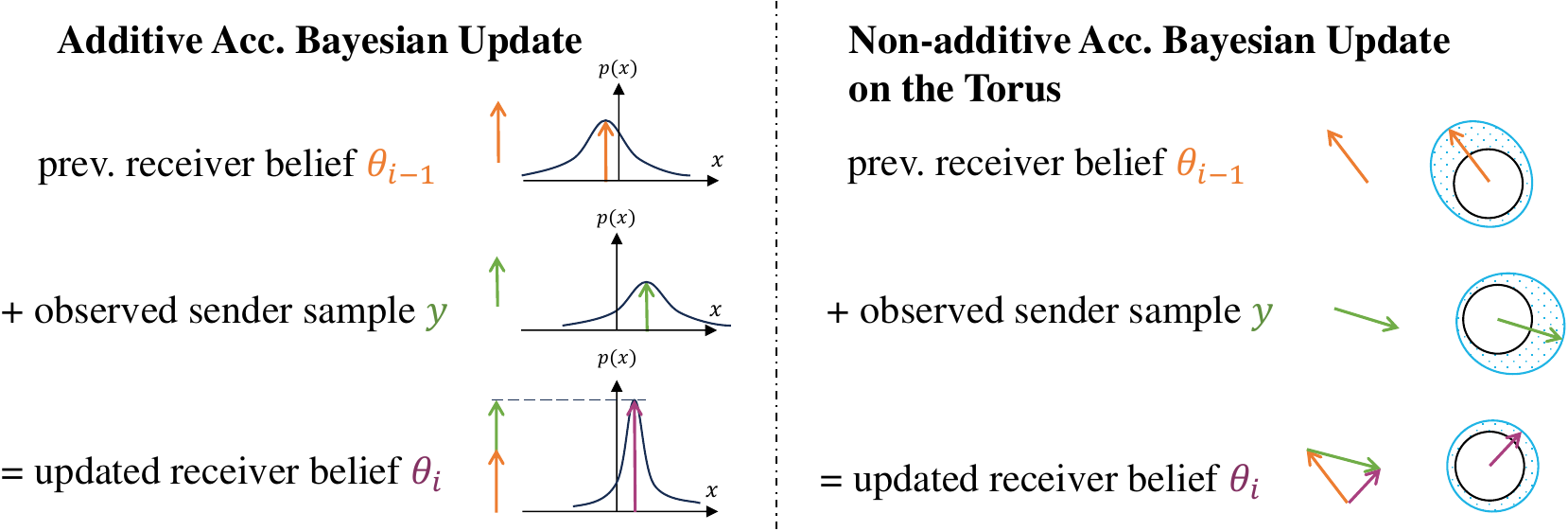}
    \caption{An intuitive illustration of non-additive accuracy Bayesian update on the torus. The lengths of arrows represent the uncertainty/entropy of the belief (\emph{e.g.}~$1/\sigma^2$ for Gaussian and $c$ for von Mises). The directions of the arrows represent the believed location (\emph{e.g.}~ $\mu$ for Gaussian and $m$ for von Mises).}
    \label{fig:non_add}
    \vskip -0.15in
\end{figure}
\textbf{Non-additive Accuracy} 
The additive accuracy is a nice property held with the Gaussian-formed sender distribution of the original BFN expressed as:
\begin{align}
\label{eq:standard_id}
    \update(\parsn{}'' \mid \parsn{}, \x; \alpha_a+\alpha_b) = \E_{\update(\parsn{}' \mid \parsn{}, \x; \alpha_a)} \update(\parsn{}'' \mid \parsn{}', \x; \alpha_b)
\end{align}
Such property is mainly derived based on the standard identity of Gaussian variable:
\begin{equation}
X \sim \mathcal{N}\left(\mu_X, \sigma_X^2\right), Y \sim \mathcal{N}\left(\mu_Y, \sigma_Y^2\right) \Longrightarrow X+Y \sim \mathcal{N}\left(\mu_X+\mu_Y, \sigma_X^2+\sigma_Y^2\right)
\end{equation}
The additive accuracy property makes it feasible to derive the Bayesian flow distribution $
p_F(\boldsymbol{\theta} \mid \mathbf{x} ; i)=p_U\left(\boldsymbol{\theta} \mid \boldsymbol{\theta}_0, \mathbf{x}, \sum_{k=1}^{i} \alpha_i \right)
$ for the simulation-free training of \cref{eq:loss_n}.
It should be noted that the standard identity in \cref{eq:standard_id} does not hold in the von Mises distribution. Hence there exists an important difference between the original Bayesian flow defined on Euclidean space and the Bayesian flow of circular data on $\mathbb{T}^D$ based on von Mises distribution. With prior $\btheta = \{\bold{0},\bold{0}\}$, we could formally represent the non-additive accuracy issue as:
\begin{align}
\label{eq:nonadd}
     &\update(c'' \mid \parsn, \x; \alpha_a+\alpha_b)  = \delta(c-\alpha_a-\alpha_b)
     \ne  \mathbb{E}_{p_U(\parsn' \mid \parsn, \x; \alpha_a)}\update(c'' \mid \parsn', \x; \alpha_b) \nonumber \\&= \mathbb{E}_{vM(\bold{y}_b|\bold{x},\alpha_a)}\mathbb{E}_{vM(\bold{y}_a|\bold{x},\alpha_b)}\delta(c-||[\alpha_a \cos\y_a+\alpha_b\cos \y_b,\alpha_a \sin\y_a+\alpha_b\sin \y_b]^T||_2)
\end{align}
A more intuitive visualization could be found in \cref{fig:non_add}. This fundamental difference between periodic Bayesian flow and that of \citet{bfn} presents both theoretical and practical challenges, which we will explain and address in the following contents.





\textbf{Entropy Conditioning} As a common practice in generative models~\citep{ddpm,flowmatching,bfn}, timestep $t$ is widely used to distinguish among generation states by feeding the timestep information into the networks. However, this paper shows that for periodic Bayesian flow, the accumulated accuracy $\vc_i$ is more effective than time-based conditioning by informing the network about the entropy and certainty of the states $\parsnt{i}$. This stems from the intrinsic non-additive accuracy which makes the receiver's accumulated accuracy $c$ not bijective function of $t$, but a distribution conditioned on accumulated accuracies $\vc_i$ instead. Therefore, the entropy parameter $\vc$ is taken logarithm and fed into the network to describe the entropy of the input corrupted structure. We verify this consideration in \cref{sec:exp_ablation}. 

\textbf{Reformulations of BFN}. Recall the original update function with Gaussian sender distribution, after receiving noisy samples $\y_1,\y_2,\dots,\y_i$ with accuracies $\senderacc$, the accumulated accuracies of the receiver side could be analytically obtained by the additive property and it is consistent with the sender side.
However, as previously mentioned, this does not apply to periodic Bayesian flow, and some of the notations in original BFN~\citep{bfn} need to be adjusted accordingly. We maintain the notations of sender side's one-step accuracy $\alpha$ and added accuracy $\beta$, and alter the notation of receiver's accuracy parameter as $c$, which is needed to be simulated by cascade of Bayesian updates. We emphasize that the receiver's accumulated accuracy $c$ is no longer a function of $t$ (differently from the Gaussian case), and it becomes a distribution conditioned on received accuracies $\senderacc$ from the sender. Therefore, we represent the Bayesian flow distribution of von Mises distribution as $p_F(\btheta|\x;\alpha_1,\alpha_2,\dots,\alpha_i)$. And the original simulation-free training with Bayesian flow distribution is no longer applicable in this scenario.

\textbf{Fast Sampling from Equivalent Bayesian Flow Distribution} Based on the above reformulations, the Bayesian flow distribution of von Mises distribution is reframed as: 
\begin{equation}\label{eq:flow_frac}
p_F(\btheta_i|\x;\alpha_1,\alpha_2,\dots,\alpha_i)=\E_{\update(\parsnt{1} \mid \parsnt{0}, \x ; \alphat{1})}\dots\E_{\update(\parsn_{i-1} \mid \parsnt{i-2}, \x; \alphat{i-1})} \update(\parsnt{i} | \parsnt{i-1},\x;\alphat{i} )
\end{equation}
Naively sampling from \cref{eq:flow_frac} requires slow auto-regressive iterated simulation, making training unaffordable. Noticing the mathematical properties of \cref{eq:h_m,eq:h_c}, we  transform \cref{eq:flow_frac} to the equivalent form:
\begin{equation}\label{eq:cirflow_equiv}
p_F(\vec{m}_i|\x;\alpha_1,\alpha_2,\dots,\alpha_i)=\E_{vM(\y_1|\x,\alpha_1)\dots vM(\y_i|\x,\alpha_i)} \delta(\vec{m}_i-\text{atan2}(\sum_{j=1}^i \alpha_j \cos \y_j,\sum_{j=1}^i \alpha_j \sin \y_j))
\end{equation}
\begin{equation}\label{eq:cirflow_equiv2}
p_F(\vec{c}_i|\x;\alpha_1,\alpha_2,\dots,\alpha_i)=\E_{vM(\y_1|\x,\alpha_1)\dots vM(\y_i|\x,\alpha_i)}  \delta(\vec{c}_i-||[\sum_{j=1}^i \alpha_j \cos \y_j,\sum_{j=1}^i \alpha_j \sin \y_j]^T||_2)
\end{equation}
which bypasses the computation of intermediate variables and allows pure tensor operations, with negligible computational overhead.
\begin{restatable}{proposition}{cirflowequiv}
The probability density function of Bayesian flow distribution defined by \cref{eq:cirflow_equiv,eq:cirflow_equiv2} is equivalent to the original definition in \cref{eq:flow_frac}. 
\end{restatable}
\textbf{Numerical Determination of Linear Entropy Sender Accuracy Schedule} ~Original BFN designs the accuracy schedule $\beta(t)$ to make the entropy of input distribution linearly decrease. As for crystal generation task, to ensure information coherence between modalities, we choose a sender accuracy schedule $\senderacc$ that makes the receiver's belief entropy $H(t_i)=H(p_I(\cdot|\vtheta_i))=H(p_I(\cdot|\vc_i))$ linearly decrease \emph{w.r.t.} time $t_i$, given the initial and final accuracy parameter $c(0)$ and $c(1)$. Due to the intractability of \cref{eq:vm_entropy}, we first use numerical binary search in $[0,c(1)]$ to determine the receiver's $c(t_i)$ for $i=1,\dots, n$ by solving the equation $H(c(t_i))=(1-t_i)H(c(0))+tH(c(1))$. Next, with $c(t_i)$, we conduct numerical binary search for each $\alpha_i$ in $[0,c(1)]$ by solving the equations $\E_{y\sim vM(x,\alpha_i)}[\sqrt{\alpha_i^2+c_{i-1}^2+2\alpha_i c_{i-1}\cos(y-m_{i-1})}]=c(t_i)$ from $i=1$ to $i=n$ for arbitrarily selected $x\in[-\pi,\pi)$.

After tackling all those issues, we have now arrived at a new BFN architecture for effectively modeling crystals. Such BFN can also be adapted to other type of data located in hyper-torus $\mathbb{T}^{D}$.

\subsection{Equivariant Bayesian Flow for Crystal}
With the above Bayesian flow designed for generative modeling of fractional coordinate $\vF$, we are able to build equivariant Bayesian flow for each modality of crystal. In this section, we first give an overview of the general training and sampling algorithm of \modelname (visualized in \cref{fig:framework}). Then, we describe the details of the Bayesian flow of every modality. The training and sampling algorithm can be found in \cref{alg:train} and \cref{alg:sampling}.

\textbf{Overview} Operating in the parameter space $\bthetaM=\{\bthetaA,\bthetaL,\bthetaF\}$, \modelname generates high-fidelity crystals through a joint BFN sampling process on the parameter of  atom type $\bthetaA$, lattice parameter $\vec{\theta}^L=\{\bmuL,\brhoL\}$, and the parameter of fractional coordinate matrix $\bthetaF=\{\bmF,\bcF\}$. We index the $n$-steps of the generation process in a discrete manner $i$, and denote the corresponding continuous notation $t_i=i/n$ from prior parameter $\thetaM_0$ to a considerably low variance parameter $\thetaM_n$ (\emph{i.e.} large $\vrho^L,\bmF$, and centered $\bthetaA$).

At training time, \modelname samples time $i\sim U\{1,n\}$ and $\bthetaM_{i-1}$ from the Bayesian flow distribution of each modality, serving as the input to the network. The network $\net$ outputs $\net(\parsnt{i-1}^\mathcal{M},t_{i-1})=\net(\parsnt{i-1}^A,\parsnt{i-1}^F,\parsnt{i-1}^L,t_{i-1})$ and conducts gradient descents on loss function \cref{eq:loss_n} for each modality. After proper training, the sender distribution $p_S$ can be approximated by the receiver distribution $p_R$. 

At inference time, from predefined $\thetaM_0$, we conduct transitions from $\thetaM_{i-1}$ to $\thetaM_{i}$ by: $(1)$ sampling $\y_i\sim p_R(\bold{y}|\thetaM_{i-1};t_i,\alpha_i)$ according to network prediction $\predM{i-1}$; and $(2)$ performing Bayesian update $h(\thetaM_{i-1},\y^\calM_{i-1},\alpha_i)$ for each dimension. 




\textbf{Bayesian Flow of Fractional Coordinate $\vF$}~The distribution of the prior parameter $\bthetaF_0$ is defined as:
\begin{equation}\label{eq:prior_frac}
    p(\bthetaF_0) \defeq \{vM(\vm_0^F|\vec{0}_{3\times N},\vec{0}_{3\times N}),\delta(\vc_0^F-\vec{0}_{3\times N})\} = \{U(\vec{0},\vec{1}),\delta(\vc_0^F-\vec{0}_{3\times N})\}
\end{equation}
Note that this prior distribution of $\vm_0^F$ is uniform over $[\vec{0},\vec{1})$, ensuring the periodic translation invariance property in \cref{De:pi}. The training objective is minimizing the KL divergence between sender and receiver distribution (deduction can be found in \cref{appd:cir_loss}): 
\begin{align}\label{loss_frac}
\calL_F = n \E_{i \sim \ui{n}, \flow(\parsn{}^F \mid \vF ; \senderacc)} \alpha_i\frac{I_1(\alpha_i)}{I_0(\alpha_i)}(1-\cos(\vF-\predF{i-1}))
\end{align}
where $I_0(x)$ and $I_1(x)$ are the zeroth and the first order of modified Bessel functions. The transition from $\bthetaF_{i-1}$ to $\bthetaF_{i}$ is the Bayesian update distribution based on network prediction:
\begin{equation}\label{eq:transi_frac}
    p(\btheta^F_{i}|\parsnt{i-1}^\calM)=\mathbb{E}_{vM(\bold{y}|\predF{i-1},\alpha_i)}\delta(\btheta^F_{i}-h(\btheta^F_{i-1},\bold{y},\alpha_i))
\end{equation}
\begin{restatable}{proposition}{fracinv}
With $\net_{F}$ as a periodic translation equivariant function namely $\net_F(\parsnt{}^A,w(\parsnt{}^F+\vt),\parsnt{}^L,t)=w(\net_F(\parsnt{}^A,\parsnt{}^F,\parsnt{}^L,t)+\vt), \forall\vt\in\R^3$, the marginal distribution of $p(\vF_n)$ defined by \cref{eq:prior_frac,eq:transi_frac} is periodic translation invariant. 
\end{restatable}
\textbf{Bayesian Flow of Lattice Parameter \texorpdfstring{$\boldsymbol{L}$}{}}   
Noting the lattice parameter $\bm{L}$ located in Euclidean space, we set prior as the parameter of a isotropic multivariate normal distribution $\btheta^L_0\defeq\{\vmu_0^L,\vrho_0^L\}=\{\bm{0}_{3\times3},\bm{1}_{3\times3}\}$
such that the prior distribution of the Markov process on $\vmu^L$ is the Dirac distribution $\delta(\vec{\mu_0}-\vec{0})$ and $\delta(\vec{\rho_0}-\vec{1})$, 
which ensures O(3)-invariance of prior distribution of $\vL$. By Eq. 77 from \citet{bfn}, the Bayesian flow distribution of the lattice parameter $\bm{L}$ is: 
\begin{align}
p_F^L(\bmuL|\bm{L};t) &=\mathcal{N}(\bmuL|\gamma(t)\bm{L},\gamma(t)(1-\gamma(t))\bm{I}) 
\end{align}
where $\gamma(t) = 1 - \sigma_1^{2t}$ and $\sigma_1$ is the predefined hyper-parameter controlling the variance of input distribution at $t=1$ under linear entropy accuracy schedule. The variance parameter $\vrho$ does not need to be modeled and fed to the network, since it is deterministic given the accuracy schedule. After sampling $\bmuL_i$ from $p_F^L$, the training objective is defined as minimizing KL divergence between sender and receiver distribution (based on Eq. 96 in \citet{bfn}):
\begin{align}
\mathcal{L}_{L} = \frac{n}{2}\left(1-\sigma_1^{2/n}\right)\E_{i \sim \ui{n}}\E_{\flow(\bmuL_{i-1} |\vL ; t_{i-1})}  \frac{\left\|\vL -\predL{i-1}\right\|^2}{\sigma_1^{2i/n}},\label{eq:lattice_loss}
\end{align}
where the prediction term $\predL{i-1}$ is the lattice parameter part of network output. After training, the generation process is defined as the Bayesian update distribution given network prediction:
\begin{equation}\label{eq:lattice_sampling}
    p(\bmuL_{i}|\parsnt{i-1}^\calM)=\update^L(\bmuL_{i}|\predL{i-1},\bmuL_{i-1};t_{i-1})
\end{equation}


\begin{restatable}{proposition}{latticeinv}\label{prop:latticeinv}
With $\net_{L}$ as  O(3)-equivariant function namely $\net_L(\parsnt{}^A,\parsnt{}^F,\vQ\parsnt{}^L,t)=\vQ\net_L(\parsnt{}^A,\parsnt{}^F,\parsnt{}^L,t),\forall\vQ^T\vQ=\vI$, the marginal distribution of $p(\bmuL_n)$ defined by \cref{eq:lattice_sampling} is O(3)-invariant. 
\end{restatable}

\textbf{Bayesian Flow of Atom Types \texorpdfstring{$\boldsymbol{A}$}{}} 
Given that atom types are discrete random variables located in a simplex $\calS^K$, the prior parameter of $\boldsymbol{A}$ is the discrete uniform distribution over the vocabulary $\parsnt{0}^A \defeq \frac{1}{K}\vec{1}_{1\times N}$. 
With the notation of the projection from the class index $j$ to the length $K$ one-hot vector $ (\oh{j}{K})_k \defeq \delta_{j k}, \text{where }\oh{j}{K}\in \R^{K},\oh{\vA}{KD} \defeq \left(\oh{a_1}{K},\dots,\oh{a_N}{K}\right) \in \R^{K\times N}$, the Bayesian flow distribution of atom types $\vA$ is derived in \citet{bfn}:
\begin{align}
\flow^{A}(\parsn^A \mid \vA; t) &= \E_{\N{\y \mid \beta^A(t)\left(K \oh{\vA}{K\times N} - \vec{1}_{K\times N}\right)}{\beta^A(t) K \vec{I}_{K\times N \times N}}} \delta\left(\parsn^A - \frac{e^{\y}\parsnt{0}^A}{\sum_{k=1}^K e^{\y_k}(\parsnt{0})_{k}^A}\right).
\end{align}
where $\beta^A(t)$ is the predefined accuracy schedule for atom types. Sampling $\btheta_i^A$ from $p_F^A$ as the training signal, the training objective is the $n$-step discrete-time loss for discrete variable \citep{bfn}: 
\begin{align}
&\calL_A = n\E_{i \sim U\{1,n\},\flow^A(\parsn^A \mid \vA ; t_{i-1}),\N{\y \mid \alphat{i}\left(K \oh{\vA}{KD} - \vec{1}\right)}{\alphat{i} K \vec{I}}} \ln \N{\y \mid \alphat{i}\left(K \oh{\vA}{K\times N} - \vec{1}\right)}{\alphat{i} K \vec{I}}\nonumber\\
&\qquad\qquad\qquad-\sum_{d=1}^N \ln \left(\sum_{k=1}^K \out^{(d)}(k \mid \parsn^A; t_{i-1}) \N{\ydd{d} \mid \alphat{i}\left(K\oh{k}{K}- \vec{1}\right)}{\alphat{i} K \vec{I}}\right)\label{discdisc_t_loss_exp}
\end{align}
where $\vec{I}\in \R^{K\times N \times N}$ and $\vec{1}\in\R^{K\times D}$. When sampling, the transition from $\bthetaA_{i-1}$ to $\bthetaA_{i}$ is derived as:
\begin{equation}
    p(\btheta^A_{i}|\parsnt{i-1}^\calM)=\update^A(\btheta^A_{i}|\btheta^A_{i-1},\predA{i-1};t_{i-1})
\end{equation}

The detailed training and sampling algorithm could be found in \cref{alg:train} and \cref{alg:sampling}.

%% file: experiment.tex
\section{Experiments}\label{sec:exp}
We evaluate on two crystal generation tasks: ab initio generation in \cref{sec:exp_abinitio} and stable structure prediction task in \cref{sec:exp_csp}. Ablation studies are detailed in \cref{sec:exp_ablation} to validate design choices. We provide the implementation details in \cref{appd:imple_details}.

Following \citet{xie2021crystal,jiao2023crystal}, we choose the following datasets for evaluation: $1)$ \textbf{Perov-5}~\citep{castelli2012new, castelli2012computational} is composed of 18,928 perovskite crystals of similar structures, with 5 atoms in a unit cell sharing the chemical formula ABX$_3$. $2)$ \textbf{Carbon-24}~\citep{carbon2020data} contains 10,153 crystals with 6$\sim$24 atoms in a cell, and all atom types are carbon. $3)$ \textbf{MP-20}~\citep{jain2013commentary} selects 45,231 stable inorganic materials from Material Projects~\citep{jain2013commentary}, including the majority of experimentally-verified materials with at most 20 atoms in a unit cell. $4)$ \textbf{MPTS-52} \citep{jiao2023crystal} consists of 40,476 crystals up to 52 atoms per cell, which is a more challenging extension of MP-20. All crystals are reduced as Niggli cells \citep{nigglicell}. The procedure to split the datasets into training, validation, and testing subsets adheres to prior practices~\citep{xie2021crystal,jiao2023crystal}. 

\subsection{Ab Initio Generation}\label{sec:exp_abinitio}
\begin{table}[t!]
  \caption{Results on ab initio generation task. Baseline results are from \citet{xie2021crystal,jiao2023crystal,flowmm}.}
  \vskip -0.1in
  \label{tab:abgen}
  \centering
  \resizebox{0.99\linewidth}{!}{
  \begin{tabular}{ccccccccc}
    \toprule
    \multirow{2}{*}{\bf Data} & \multirow{2}{*}{\bf Method}  & \multicolumn{2}{c}{\bf Validity (\%)  $\uparrow$}  & \multicolumn{2}{c}{\bf Coverage (\%) $\uparrow$}  & \multicolumn{3}{c}{\bf Property $\downarrow$}  \\
    ~     & ~     & Struc. & Comp. & COV-R & COV-P & $d_\rho$ & $d_E$ & $d_{elem}$ \\
    \midrule
    \multirow[t]{5}{*}{Perov-5} 
    & Cond-DFC-VAE~\citep{court20203} & 73.60 & 82.95 & 73.92 & 10.13 & 2.268 & 4.111 & 0.8373 \\
    & G-SchNet~\citep{NIPS2019_8974}  & 99.92 & 98.79 & 0.18 & 0.23 & 1.625 &	4.746 &	0.0368 \\
    & P-G-SchNet~\citep{NIPS2019_8974} & 79.63 & \textbf{99.13} & 0.37 & 0.25 & 0.2755 & 1.388 & 0.4552 \\
    & CDVAE~\citep{xie2021crystal} & \textbf{100.0} & 98.59 & 99.45 & 98.46 & 0.1258 & 0.0264 & 0.0628 \\
    & DiffCSP\citep{jiao2023crystal}  & \textbf{100.0} & 98.85 & \textbf{99.74} & 98.27 & 0.1110 & 0.0263 & 0.0128 \\
    & \modelname & \textbf{100.0} & 98.86 & 99.52 & \textbf{98.63} & \textbf{0.0728} & \textbf{0.0198} & \textbf{0.0098} \\
    \midrule
    \multirow[t]{5}{*}{Carbon-24} 
    & G-SchNet~\citep{NIPS2019_8974} & 99.94	 & --	  & 0.00 & 0.00 & 0.9427 & 1.320 &	-- \\
    & P-G-SchNet~\citep{NIPS2019_8974} & 48.39 & -- & 0.00 & 0.00 & 1.533 & 134.7 & --  \\
    & CDVAE~\citep{xie2021crystal} & \textbf{100.0} & -- & 99.80 & 83.08 & 0.1407 & 0.2850 & -- \\
    & DiffCSP \citep{jiao2023crystal}     & \textbf{100.0} & -- & \textbf{99.90} & 97.27 & 0.0805 & 0.0820 & --  \\
    & \modelname   & \textbf{100.0} & -- & \textbf{99.90} & \textbf{99.12} & \textbf{0.0612} & \textbf{0.0503} & --  \\
    \midrule
    \multirow[t]{5}{*}{MP-20} 
    & G-SchNet~\citep{NIPS2019_8974} & 99.65 &	75.96 & 38.33 & 99.57 &	3.034 &	42.09 &	0.6411	\\
    & P-G-SchNet~\citep{NIPS2019_8974} & 77.51 & 76.40 & 41.93 & 99.74 & 4.04 & 2.448 & 0.6234 \\
    & CDVAE~\citep{xie2021crystal} & \textbf{100.0} & 86.70 & 99.15 & 99.49 & 0.6875 & 0.2778 & 1.432  \\
    & DiffCSP\citep{jiao2023crystal} &  \textbf{100.0} & 83.25 & \textbf{99.71} & 99.76 & 0.3502 & 0.1247 & 0.3398  \\
    & FlowMM \citep{flowmm} &  96.85 & 83.19 & 99.49 & 99.58 & 0.239 & - & \textbf{0.083}  \\
     & \modelname & \textbf{100.0} & \textbf{87.51} & 99.09 & \textbf{99.79} & \textbf{0.2067} & \textbf{0.0632} & 0.1628\\
    \bottomrule
  \end{tabular}
  }
\end{table}

\begin{table*}[t!]
\vskip -0.1in
  \centering
  \caption{Results on stable structure prediction task. Baseline results are from \citet{jiao2023crystal,flowmm}.}
  \resizebox{0.90\linewidth}{!}{
    \setlength{\tabcolsep}{2.5pt}
    \begin{tabular}{lccccccc}
    \toprule
    \multirow{2}[2]{*}{} & \multicolumn{2}{c}{Perov-5 } &  \multicolumn{2}{c}{MP-20} & \multicolumn{2}{c}{MPTS-52} \\
          &  Match rate$\uparrow$ & RMSE$\downarrow$   & Match rate$\uparrow$ & RMSE$\downarrow$ & Match rate$\uparrow$ & RMSE$\downarrow$ \\
    \midrule
    CDVAE~\citep{xie2021crystal}   & 45.31  & 0.1138  & 33.90  & 0.1045 & 5.34 & 0.2106 \\
    \midrule
    DiffCSP~\citep{jiao2023crystal}      & 52.02  & 0.0760  & 51.49 &	0.0631 & 12.19 & 0.1786 \\
    \midrule
    FlowMM~\citep{flowmm}     & 53.15  & 0.0992  & 61.39 &	 0.0566 & 17.54 &  0.1726 \\
    \midrule
    \modelname    & \textbf{54.69}  & \textbf{0.0636} & \textbf{64.35} &	\textbf{0.0433} & \textbf{20.52} & \textbf{0.1038} \\
    \bottomrule
    \end{tabular}
    }
\vskip -0.25in
 \label{tab:oto}%
\end{table*}%

\textbf{Baselines} For this task, the compared baselines include:  $2)$ two-stage VAE-based methods \textbf{Cond-DFC-VAE} \citep{court20203} and \textbf{CD-VAE} \citep{xie2021crystal}; $2)$ auto-regressive method \textbf{G-SchNet}~\citep{NIPS2019_8974}, and its periodic adaptation \textbf{P-G-SchNet}~\citep{xie2021crystal}; $3)$ diffusion-based joint generation approach \textbf{DiffCSP} \citep{jiao2023crystal}. $4)$ flow-matching-based approach \textbf{FlowMM} (Note that FlowMM only reports results on MP-20 and excludes $d_{E}$). We follow \citep{hoogeboom2022equivariant,jiao2023crystal} to sample atom numbers from a distribution that is pre-computed based on atom numbers in the training dataset. \textbf{Performance Indicators} Following previous work \citep{xie2021crystal}, we evaluate the efficacy of our model from three aspects: $1)$ \textit{Validity}: Structure and compositional validity of randomly generated 10000 materials. $2)$ \textit{Coverage}: Coverage score between generated 10000 materials and test set, defined by average minimum structure distance and average minimum compositional distance. $3)$ \textit{Property Statistics}: the earth mover's distance (EMD) between the property distribution of generated crystals and test dataset crystals. Monitored properties include density ($\rho$, unit g/cm$^3$), energy predicted by an independent GNN ($E$, unit eV/atom), and the number of unique elements (\# elem.).

\textbf{Results} The evaluation metrics for ab initio generation tasks are listed in \cref{tab:abgen}. Our method consistently achieves better or competitive property statistics and generation precision on three datasets compared to baseline generative models. For compositional metrics including $d_{elem}$ and compositional validity, our method demonstrates bigger performance improvement for the more challenging dataset MP-20 ($+4.34$\% compared to DiffCSP with the same level of $d_{elem}$), underscoring the importance of modeling atom types in the simplex space. 

\subsection{Stable Structure Prediction} \label{sec:exp_csp}
In this section, we extend our method to stable structure prediction task, where the modeling target is $p(\vL,\vF|\vA)$. The condition of atom types $\vA$ is incorporated into the network by concatenating node feature and atom type embedding, following \citet{jiao2023crystal}. 

\textbf{Baselines}~Following the practices in \citet{jiao2023crystal}, we select baselines generative approaches including Diffusion-based approaches \textbf{CD-VAE} and \textbf{DiffCSP} and the recent flow-matching-based method, \textbf{FlowMM}, which only reports results on the MP-20 dataset for this task. \textbf{Performance Indicators}~The measured performance indicators for this task are \textbf{Match Rate} and \textbf{RMSD} computed by \verb|StructureMatcher| class with thresholds \verb|stol=0.5, angle_tol=10, ltol=0.3| in pymatgen~\citep{ong2013python}, between the predicted 
 structure candidates and the groundtruth structure given the composition. \textbf{Results} As summarized in \cref{tab:oto},  \modelname achieves consistent performance improvement over baseline methods, especially for more challenging datasets ($\sim$ 13$\%$ higher match rate than DiffCSP for MP-20 and $\sim$ 40$\%$ lower RMSE compared to FlowMM for MPTS-52).

\begin{wraptable}{r}{0.50\textwidth}
 \vskip -0.3in
\small
  \centering
  \caption{Ablation studies on MP-20. }
    \begin{tabular}{lcc}
    \toprule
          & Match rate (\%) $\uparrow$ & RMSE $\downarrow$  \\
    \midrule
    \modelname & \textbf{64.35} & \textbf{0.0433} \\
    \midrule
    $w/o$~entropy cond. & 52.16 & 0.0631 \\
    $w/o $ approx. sch. & 49.76 & 0.0643 \\
    $w/o$ torus BFN & 6.17 & 0.3822\\
    \midrule
     & \multicolumn{2}{c}{1k Batches Sim. Time (s)} \\
    \midrule
    Iterated Sim.  & \multicolumn{2}{c}{356.1}\\
    Fast Sim. & \multicolumn{2}{c}{92.6}\\
    \bottomrule
    \end{tabular}%
  \label{tab:exp_abl}%

\end{wraptable}

\subsection{Ablation Study}\label{sec:exp_ablation}

Using MP-20 dataset and stable structure prediction task, we validate the necessity of proposed components of CrysBFN with results summarized in \cref{tab:exp_abl}: 

$1)$ By removing the entropy parameter condition $\bcF$ and using time as condition, match rate drops to $52.16\%$, proving that, different from original BFN, the non-additive accuracy dynamics requires the network to model both mean parameter $m$ and entropy parameter condition $c$. 

$2)$ By altering the approximated linear-entropy sender accuracy schedule to the hand-designed roughly linear-entropy schedule $c(t)=t c^{t}(1)$, we validate the effect of exact searched linear entropy schedule. 

$3)$ By replacing the proposed hyper-torus BFN with the original continuous BFN, we observe poor match rate at $6.17\%$, indicating the importance of redesigning BFN for crystal data. Calculating the computational time for simulating 1000 batches, we observe $\sim4\times$ efficiency, verifying fast sampling rate considering the full training procedure of MP-20 requires $\sim 150$k steps.

\subsection{Sampling Efficiency Experiment}
\begin{wrapfigure}{r}{0.5\textwidth} 
\vskip -0.4in
    \centering
    \includegraphics[width=0.4\textwidth]{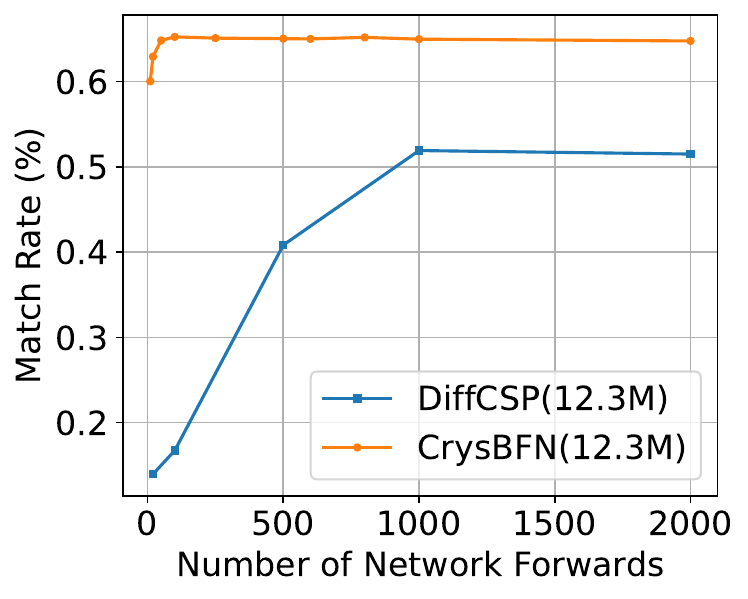} 
    \caption{Experimental results on MP-20 with different Number of Function Evaluations (NFE) \emph{i.e.} number of network forward passes. } 
    \label{fig:eff_compare}
\vskip -0.9in
\end{wrapfigure}
We compare the sampling efficiency of CrysBFN and DiffCSP over the CSP task on the MP-20 dataset, based on the Number of Function Evaluations (NFE), \textit{i.e.}, the number of network forward passes. The experimental results is plotted in \cref{fig:eff_compare}. Notably, CrysBFN achieves a remarkable match rate of 60.02\% with only 10 step network forwards, surpassing DiffCSP's performance of 51.49\% at 2000 step network forwards. This illustrates the exceptional sampling efficiency of CrysBFN.

%% file: appendix.tex
\appendix
\clearpage
\section{Bayesian Flow Networks for Circular Data}\label{appd:bfn_cir}
In this section, we provide a detailed derivation of Bayesian flow networks considering periodicity. 
\subsection{Circular Data and von Mises Distribution}\label{appd:vm_intro}
\begin{figure}
    \centering
    \includegraphics[width=0.8\textwidth]{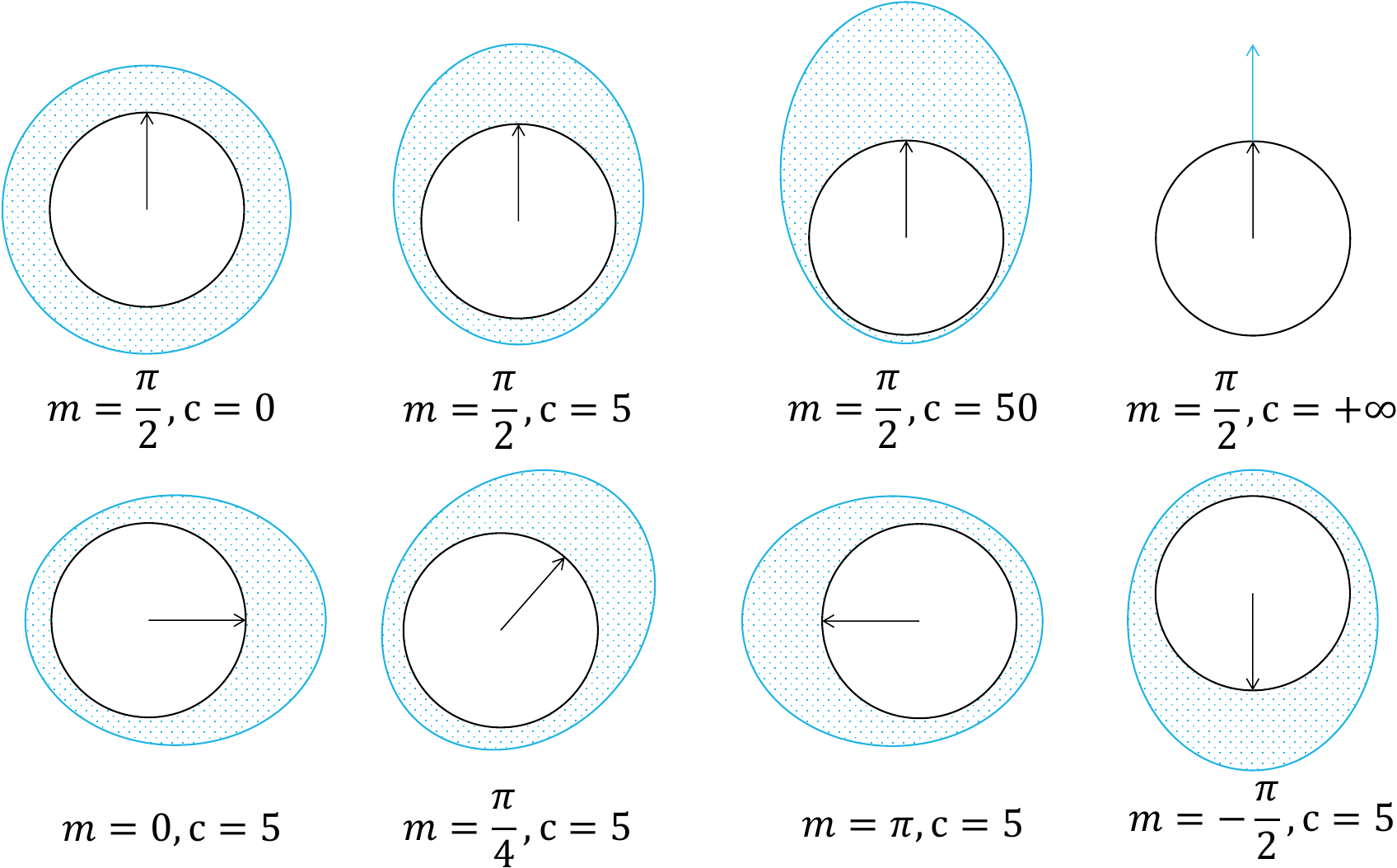}
    \caption{Depiction of von Mises distributions with different directions parameters $m$ and concentration parameters $c$. The parameter $m$ denotes the central location about which the distribution is centered, while $c$ functions as a measure of the distribution's concentration. When $c = 0$, the distribution is uniform on the circle. As $c$ increases, the distribution becomes more concentrated around the value $m$, with $c$ serving as a measure of this concentration. In the limit as $c \rightarrow +\infty$, the distribution converges to $\delta(m)$, a Dirac delta distribution centered at $m$.  }
    \label{fig:vm_alpha}
\end{figure}
One-dimensional circular data $x$ refers to observations of random variables supporting on the circumference of the unit circle defined in directional statistics \citep{mardia2009directional,ley2017modern}. This space can be represented by the torus:
\begin{equation}
     \mathbb{T}^1\defeq \{\vz\in\mathbb{R}^2:||\vz||=1\}
\end{equation}

For $n$-dimensional data $\vx\in \mathbb{R}^n$, the set of $\vx$ with every dimension located in $\mathbb{T}^1$ form a compact Riemannian manifold named hyper-torus $\mathbb{T}^n$ formally.

Wrapped normal distribution used in \citet{jiao2023crystal} and von Mises distribution used in this paper are both circular distributions defined in this space, the probability density function of von Mises distribution with mean direction parameter $m$ and concentration parameter $c$ is 
\begin{equation}
f(x|m,c)=vM(x|m,c)=\frac{\exp(c\cos(x-m))}{2\pi I_0(c)}
\end{equation}
where $I_0(c)$ is the modified Bessel function of the first kind of order 0 as the normalizing constant. The parameters $m$ and $1/c$ are analogous to mean $\mu$ and variance $\sigma^2$ in the normal distribution: 1) $m$ represents the central location around which the distribution is clustered, while $c$ serves as a measure of concentration. 2) We give a depiction of von Mises distributions with different directions parameters in \cref{fig:vm_alpha}. When $c$ equals zero, the distribution is uniform. As $c$ becomes large, the distribution becomes tightly concentrated around the value $m$, with $c$ quantifying this concentration. In the limit as $c\rightarrow+\infty$, the distribution becomes a Dirac delta distribution centered at $m$. Its support can be chosen as any interval of length $2\pi$ and we in this paper choose $[-\pi,\pi)$. Note that the fractional coordinate can be transformed to this interval easily by a linear transformation $g(x)=2\pi x-\pi$. For this modeled interval, the map function from $\R$ to $[-\pi,\pi)$ is 
\begin{equation}
w_{[-\pi,\pi)}(x)=(x-\pi)\%2\pi-\pi=x+2\pi k, \exists k\in \mathbb{Z}
\end{equation}
Such map function is equivalent to the map function $w(x)=x-\lfloor x\rfloor=x+k, \exists k\in \mathbb{Z}$ used in \citet{jiao2023crystal} if we choose the modelled interval as $[0,1)$. Then, we can prove that the probability density distribution of von Mises distribution is equivariant to the periodic translation transformation:
\begin{align}
    \forall t\in \R,\quad f(w_{[-\pi,\pi)}(x+t)|w_{[-\pi,\pi)}(m+t),c)&=f(x+2\pi k'|m+2\pi k+t,c)\nonumber\\
    &=\frac{\exp(c\cos(x+2\pi k'-(m+2\pi k+t)))}{2\pi I_0(c)}\nonumber\\
    &=\frac{\exp(c\cos((x-m))}{2\pi I_0(c)}\nonumber\\
    &=f(x|m,c) \label{eq:vm_period_equi}
\end{align}
The differential entropy of von Mises distribution with mean direction parameter $m$ and concentration parameter $c$ is
\begin{equation}\label{eq:vm_entropy}
    H(vM(x|m,c))=-c\frac{I_1(c)}{I_0(c)}+\ln[2\pi I_0(c)]\quad\text{\citep{mardia2009directional}}
\end{equation}

We opt for von Mises distribution rather than the wrapped normal distribution used in \citet{jiao2023crystal,jing2022torsional} mainly because of the Bayesian conjugacy of von Mises distribution the posterior of which is conjugate to the prior distribution if the likelihood is parameterized as von Mises distribution which is the fundamental basis of constructing a Bayesian flow. Interestingly, there is an intriguing connection between von Mises distribution and crystal force field that the von Mises distribution is the stationary distribution of a drift and diffusion process on the circle in a harmonic potential corresponding to the harmonic force field of crystals \citep{risken1996fokker}.
\subsection{Input Distribution \texorpdfstring{$p_I(\ \cdot\ |\bm{\theta})$}{} and Sender Distribution \texorpdfstring{$p_S(\vy | \vx;\alpha$)}{}}
For circular data $\bold{x}$ which locates in a quotient space $\mathbb{T}^n=\mathbb{R}^{3\times N}/[-\pi,\pi)^{3\times N}$, we define the input distribution of the Bayesian Flow Networks as independently factorized von Mises distribution over the interval $[-\pi,\pi)$.
\begin{equation}
    \btheta \defeq \{\bold{m},\bold{c}\}
\end{equation}
\begin{equation}
    p_I( \bold{x} | \theta)\defeq\Pi_{d=1}^D vM(x^{(d)}|m^{(d)},c^{(d)})
\end{equation}
where $m^{(d)}\in [-\pi,\pi)$ and $c^{(d)}\in [0,\infty)$.

In this paper, to ensuring the periodic translational invariance, the prior parameter of \modelname's Bayesian flow is chosed as 
\begin{equation}
p(\bthetaF_0) \defeq \{vM(\vm_0|\vec{0}_{3\times N},\vec{0}_{3\times N}),\delta(\vc_0-\vec{0}_{3\times N})\} = \{U(\vec{0},\vec{1}),\delta(\vc_0-\vec{0}_{3\times N})\}
\end{equation}
where $\bold{0}$ is the length of $D$ vector whose entries are all $0$. Note that this input prior $\btheta_0$ defines a multivariate uniform distribution
\begin{equation}
    p_I( \bold{x} | \btheta_0)=\Pi_{d=1}^D vM(x^{(d)}|0,0)=\Pi_{d=1}^D U(-\pi,\pi)
\end{equation}

which ensures \textit{periodic E(3) invariance} of the prior distribution. 

The sender space $\mathcal{Y}$ is identical to the data space $\mathcal{X}$ for circular data. And the sender distribution is von Mises distribution centered on $\vx$ with concentration parameter $\alpha$ represented by
\begin{equation}
    p_S (\bold{y} |\bold{x};\alpha) =\Pi_{d=1}^{D}vM(y^{(d)}|x^{(d)},\alpha):=vM(\bold{y}|\bold{x},\alpha)
\end{equation}
\subsection{Bayesian Update Function \texorpdfstring{$h(\parsnt{i-1}, \y, \alpha)$}{} and Bayesian Update Distribution \texorpdfstring{$\update(\cdot \mid \parsn, \x; \alpha)$}{}}\label{apdx:bayesian_update_function}
For the receiver, given his last univariate belief parameterized by von Mises distribution with parameter $\theta_{i-1}=\{m_{i-1},\ c_{i-1}\}$, he now observes a sample $y$ from sender distribution with unknown $x$ and known $\alpha$. Now, by Bayesian theorem,
\begin{align}
    p(x|y;\alpha,m_{i-1},c_{i-1})&=\frac{p(y|x;\alpha)p(x;m_{i-1},c_{i-1})}{p(y)}\\
    &\propto p(y|x;\alpha)p(x;m_{i-1},c_{i-1})\\
    &= vM(y|x,\alpha)vM(x|m_{i-1},c_{i-1})\\
    &\propto \exp\{{\alpha \cos(x-y)+c_{i-1}\cos(x-m_{i-1})}\} 
\end{align}
The last expression has the form of a von Mises distribution in $x$ and hence:
\begin{equation}
    p(x|y;\alpha,m_{i-1},c_{i-1})=vM(x;m_i,c_i)
\end{equation}
where 
\begin{equation}
m_i=\text{atan2}(\alpha\sin y+c_{i-1}\sin m_{i-1}, {\alpha\cos y+c_{i-1}\cos m_{i-1}})
\end{equation}
\begin{equation}
c_i =\sqrt{\alpha^2+c_{i-1}^2+2\alpha c_{i-1}\cos(y-m_{i-1})}
\end{equation}
We refer readers interested in more detailed deduction to \citet{mardia1976bayesian,guttorp1988finding}. Defining the notation for scaler $x$ in circular space as $\dotx\defeq[\cos x,\sin x]^T$, these two expressions will be much simpler, intuitive and more similar to the Gaussian form:
\begin{equation}
    h(\{\dotm_{i-1},c\},\doty,\alpha)=\{\dotm_i,c_i \}
\end{equation}
where
\begin{equation}\label{eq:update_mi}
\dotm_{i}=\frac{\alpha \doty+c_{i-1}\dotm_{i-1}}{c_i}
\end{equation}
\begin{equation}\label{eq:update_ci}
c_i=||\alpha\doty+c_{i-1}\dotm_{i-1}||_2
\end{equation}

 The Bayesian update distribution $\update(\cdot \mid \parsn, \x; \alpha)$ could be obtained by marginalizing $\y$:
\begin{equation}
p_U(\theta'|\theta,\bold{x};\alpha)=\mathbb{E}_{p_S(\bold{y}|\bold{x};\alpha)}\delta(\theta'-h(\theta,\bold{y},\alpha))=\mathbb{E}_{vM(\bold{y}|\bold{x},\alpha)}\delta(\theta'-h(\theta,\bold{y},\alpha))
\end{equation}

\subsection{\textit{Non-Additive} Accuracy Issue }
Although all cases including continuous and discrete data are proven to enjoy the so-called additive accuracy property considered in \citet{bfn} defined as:
\begin{align}
\update(\parsn'' \mid \parsn, \x; \alpha_a+\alpha_b) = \E_{\update(\parsn' \mid \parsn, \x; \alpha_a)} \update(\parsn'' \mid \parsn', \x; \alpha_b)\label{eq:additive}.
\end{align}
this property does not hold for von Mises distribution. The untenability of this property for von Mises distribution can be checked out by considering two steps Bayesian updates with prior $\btheta = \{\bold{0},\bold{0}\}$, $\alpha_a,\alpha_b, \y_a, \y_b$,
\begin{align}
    & \update(c'' \mid \parsn, \x; \alpha_a+\alpha_b)  = \delta(c-\alpha_a-\alpha_b) \\
     \ne & \mathbb{E}_{p_U(\parsn' \mid \parsn, \x; \alpha_a)} \update(c'' \mid \parsn', \x; \alpha_b) = \mathbb{E}_{vM(\bold{y}|\bold{x},\alpha_a)}\mathbb{E}_{vM(\bold{y}|\bold{x},\alpha_b)}\delta(c-||\alpha_a \y_a+\alpha_b \y_b||)
\end{align}
Consequently, the Bayesian flow distribution does not equal to one-step Bayesian update distribution with $\beta(t)$:
\begin{align}
\flow(\parsn \mid \x ; t) \ne \update(\parsn \mid \parsnt{0}, \x; \beta(t))\label{param_flow_dist}.
\end{align}
With an accuracy schedule $\alpha_1,\alpha_2,\dots,\alpha_n$ and $\beta(t)=\sum_{j=1}^{t_i} \alpha_j$, this untenability will cause the incongruity between sender's accumulated accuracy $\beta(t)$ and the confidence of receiver's belief $c_i$ over his location parameter $m_i$. Hence we should differentiate the sender's accuracy schedule $\alpha_i$ and the receiver's belief confidence$c_i$. 
And $c_i$ is no longer a function but a \textit{distribution} over $t_i$. In consequence, we should define the Bayesian flow distribution parameterized by received sender's accuracies $\alpha_1,\alpha_2,\dots,\alpha_i$ rather than $t$. Furthermore, the information of receiver confidence $c_i$ should be part of the network input as well. 

\subsection{Bayesian Flow Distribution \texorpdfstring{$p_F(\btheta|\x;\alpha_1,\alpha_2,\dots,\alpha_i)$}{}  and Sender Accuracy Schedule \texorpdfstring{$\alpha_i$}{}}
\begin{equation}\label[equation]{eq:cirflow1}
    p_F(\btheta|\x;\alpha_1,\alpha_2,\dots,\alpha_i)=\E_{\update(\parsnt{1} \mid \parsnt{0}, \x ; \alphat{1})}\dots\E_{\update(\parsn \mid \parsnt{i-1}, \x ; \alphat{i})} \update(\parsnt | \parsnt{i-1},\x;\alphat{i} )
\end{equation}
The original definition of Bayesian flow distribution in \cref{eq:cirflow1} provides an iterative algorithm to sample from $p_F$ but which practically is slow resulting the training unaffordable. In fact, noticing the "additive" property of $c_i\dotm_i$ by \cref{eq:update_mi,eq:update_ci}, we can sample from $p_F$ without iteration:
\begin{equation}\label{eq:appd_cirflow_equiv}
p_F(\vec{m}|\x;\alpha_1,\alpha_2,\dots,\alpha_i)=\E_{vM(\y_1|\x,\alpha_1)}\dots\E_{vM(\y_i|\x,\alpha_i)} \delta(\vec{m}-\text{atan2}(\sum_{j=1}^i \alpha_i \cos \y_i,\sum_{j=1}^i \alpha_i \sin \y_i))
\end{equation}
\begin{equation}\label{eq:appd_cirflow_equiv2}
p_F(\vec{c}|\x;\alpha_1,\alpha_2,\dots,\alpha_i)=\E_{vM(\y_1|\x,\alpha_1)}\dots\E_{vM(\y_i|\x,\alpha_i)} \delta(\vec{c}-||(\sum_{j=1}^i \alpha_i \cos \y_i,\sum_{j=1}^i \alpha_i \sin \y_i)||_2)
\end{equation}
\cref{eq:appd_cirflow_equiv,eq:appd_cirflow_equiv2} provides an algorithm allowing sampling from $p_F$ by pure tensor operations without simulating the flow iteratively. Next, we can define the entropy of the receiver's belief as $H(t)$:
\begin{align}
    H(t)&\defeq \E_{p_F(\btheta|\x;\alpha_1,\alpha_2,\dots,\alpha_i)}H(p_I(~\cdot|\btheta))\\
&=\E_{p_F(c_i|\x;\alpha_1,\alpha_2,\dots,\alpha_i)}-c_i\frac{I_1(c_i)}{I_0(c_i)}+\ln[2\pi I_0(c_i)], \text{where}~i=nt
\end{align}
To ensure the information coherence between modalities, we choose to find a sender accuracy schedule to make the receiver's belief entropy $H(t)$ linearly decrease with predefined $c_n$. Formally, we would like to find an accuracy schedule $\alpha_i$ such that
\begin{equation}\label{eq:linear_entropy}
    H(t)=(1-t)H(0)+tH(1)=\ln 2\pi -c_n\frac{I_1(c_n)}{I_0(c_n)}+\ln[2\pi I_0(c_n)]
\end{equation}
Note that \cref{eq:linear_entropy} can not be solved analytically but we can solve it numerically by firstly getting the target sender's accumulated accuracy $\beta(t)$ via binary search due to the monotonicity of \cref{eq:vm_entropy}. Next, we could iteratively search $\alpha_i$ from $i=1$ to $i=n$ by matching the average accuracy toward $c_i$. This process could be done only once and the resultant $\alpha_i$ can be cached for each pre-confirmed hyper-parameter $(c_n, n)$ pair.

\subsection{Output Distribution \texorpdfstring{$\out(\cdot \mid \parsn;t)$}{} and Receiver Distribution \texorpdfstring{$\rec(\cdot \mid \parsn; \alpha,t)$}{}}
Given samples $\btheta=\{\vm,\vc\}$ from Bayesian flow distribution as input, the receiver uses network output $\net(\btheta,t)$ to rebuild his belief over ground truth $\vx$ termed output distribution. Following \citet{bfn}, we parameterize $\out$ using $\bold{\hat{x}}(\boldsymbol\theta,t)=\net(\btheta,t)$ to be $\delta$ prediction of $\x$:
\begin{equation}
    p_O(\bold{x}|\boldsymbol\theta;t)=\delta(\bold{x}-\bold{\hat{x}}(\boldsymbol\theta,t))
\end{equation}
Therefore, the receiver distribution is:
\begin{equation}
p_R(\bold{y}|\boldsymbol\theta;\alpha,t)=\mathbb{E}_{p_O(\bold{x}'|\theta;t)}p_S(\bold{y}|\bold{x}';\alpha)=vM(\bold{y}|\hat{\bold{x}}(\boldsymbol\theta,t),\alpha)
\end{equation}

\subsection{Discrete-Time Loss \texorpdfstring{$L^{n}(\x)$}{}}\label{appd:cir_loss}
From \citet{kitagawa2022kldvm}, the KL divergence between $vM(m_1,c_1)$ and $vM(m_2,c_2)$ is
\begin{equation}\label{eq:kld_vm}
    D_{KL}(vM(m_1,c_1)||vM(m_2,c_2))=-\ln\frac{I_0(c_1)}{I_0(c_2)}+\frac{I_1(c_1)}{I_0(c_1)}(c_1\dotm_1-c_2\dotm_2)'\dotm_1
\end{equation}
From \cref{eq:loss_n}, the discrete-time loss for circular data is
\begin{align}\label{disc_t_loss_exp}
L^{n}(\x) &= n \E_{i \sim \ui{n}, \flow(\parsn \mid \x ; \senderacc)} \kl{\sender{\cdot}{\x ; \alphat{i}}}{\rec(\cdot \mid \parsn; t_{i-1}, \alphat{i})}\\
&=n \E_{i \sim \ui{n}, \flow(\parsn \mid \x ; \senderacc)} \frac{I_1(\alpha_i)}{I_0(\alpha_i)}\alpha_i(1-\cos(\bold{x}-\hat{\bold{x}}(\theta_{i-1},t_{i-1}))
\end{align}
Continuous-time loss is not tractable because the Bayesian flow distribution is not analytical due to the non-additive accuracy property.

\section{Proof of Propositions}\label{appd:geometric_invar}
In this section, we first prove the crystal geometric invariance of \modelname. Crystals remain the same under transformations including permutation, orthogonal transformation, and periodic translation defined as follows:
\begin{definition} [Permutation Invariance \citep{jiao2023crystal}]
\label{De:pi}
For any permutation matrix $\mP$, $p(\mL,\mF,\mA)=p(\mL,\mF\mP, \mA\mP)$, \emph{i.e.}, changing the order of atoms will not change the distribution.  
\end{definition}
\begin{definition} [O(3) Invariance \citep{jiao2023crystal}]
\label{De:oi}
For any orthogonal transformation $\mQ\in\R^{3\times 3}$ satisfying $\mQ^\top\mQ=\mI$, $p(\mQ\mL,\mF,\mA)=p(\mL,\mF,\mA)$, namely, any rotation/reflection of $\mL$ keeps the distribution unchanged.   
\end{definition}
\begin{definition}[Periodic Translation Invariance \citep{jiao2023crystal}] 
\label{De:PTI}
For any translation $\vt\in\R^{3\times1}$, $p(\mL, w(\mF + \vt\vone^\top),\mA)=p(\mL, \mF ,\mA)$, where the function $w(\mF)=\mF - \lfloor\mF\rfloor \in [0,1)^{3\times N}$ returns the fractional part of each element in $\mF$, and $\vone\in\R^{3\times1}$ is a vector with all elements set to one. It explains that any periodic translation of $\mF$ will not change the distribution.  
\end{definition}
The combination of the above invariances is abbreviated in a compact manner termed \textit{periodic E(3) invariance} proposed by \citet{jiao2023crystal}. The permutational invariance can be easily achieved by using GNN frameworks. Periodic translation and rotation are both space group transformations. We first introduce the basic concept of \emph{G-invariant}.
\begin{definition}
\label{def:ei}
We call a distribution $p(x)$ is $G$-invariant if for any transformation $g$ in the group $G$, $p(g\cdot x) = p(x)$, and a conditional distribution $p(x|c)$ is G-equivariant if $p(g\cdot x|g\cdot c) = p(x|c), \forall g\in G$.
\end{definition}
With a lemma from \citet{xu2021geodiff}, we can prove a Markov-process-generated distribution \emph{G-invariant} by proving the G-invariance of the prior distribution and the G-equivariance of every transition kernel.
\begin{lemma}[\citet{xu2021geodiff}]
\label{lm:mrkv}
Consider the generation Markov process $p(\theta_n) =\int  p(\theta_0)p(\theta_{n:1}|\theta_0)d\theta_{1:n}$. If the prior distribution $p(\theta_0)$ is G-invariant and the Markov transitions $p(\theta_{t+1}|\theta_t), 0 \leq t\leq n-1$ are G-equivariant, the marginal distribution $p(\theta_n)$ is also G-invariant.
\end{lemma}
\begin{proof}
\begin{align*}
\forall g\in G,\quad p(g\cdot \theta_n) & = p(g\cdot \theta_0)\int p(\theta_{n:1}|\theta_0)d\theta_{1:n}\\
    & = p(g\cdot \theta_0)\int\prod_{t=0}^{n-1} p(g\cdot \theta_{t+1}|g\cdot \theta_{t})d\theta_{1:n}\\
    & = p(\theta_0)\int\prod_{t=0}^{n-1}p(g\cdot \theta_{t+1}|g\cdot \theta_t)d\theta_{1:n}\\
    & = p(\theta_0)\int\prod_{t=0}^{n-1}p(\theta_{t+1}|\theta_t)d\theta_{1:n}\\
    & = p(\theta_0)\int p(\theta_0)p(\theta_{n:1}|\theta_0)d\theta_{1:n}\\
    & = p(\theta_n).
\end{align*}
Therefore, the marginal distribution $p(\theta_n)$ is G-invariant.
\end{proof}
With \cref{lm:mrkv}, we can prove the following propositions mentioned in the main text:
\latticeinv*
\begin{proof}
    The prior is O(3) invariant since $p(\vmu^L_0)=\delta(\vmu-\Vec{0})=\delta(\vQ\vmu-\Vec{0}),\forall \vQ^T\vQ=\vI$.\\
    The transition probability of
    \begin{align*}
    &p(\vQ\bmuL_{i}|\vQ\bmuL_{i-1},\parsnt{i-1}^A,\parsnt{i-1}^F)\\
    =&\update^L(\vQ\bmuL_{i}|\vQ\bmuL_{i-1},\hat{\net}_L(\vQ\bmuL_{i-1},\cdot),t_{i-1})\\
    =&\mathcal{N}(\vQ\bmuL_{i}|\frac{\alpha\hat{\net}_L(\vQ\bmuL_{i-1},\cdot)+\vQ\bmuL_{i-1}\rho_{i-1}}{\rho_{i}},\frac{\alpha}{\rho_i^2}\vI) \\
    =&\mathcal{N}(\vQ\bmuL_{i}|\frac{\alpha\vQ\hat{\net}_L(\bmuL_{i-1},\cdot)+\vQ\bmuL_{i-1}\rho_{i-1}}{\rho_{i}},\frac{\alpha}{\rho_i^2}\vI) \quad (\text{by equivariance of }\net^L)\\
    =&\mathcal{N}(\vQ\bmuL_{i}|\vQ\frac{\alpha\hat{\net}_L(\bmuL_{i-1},\cdot)+\bmuL_{i-1}\rho_{i-1}}{\rho_{i}},\frac{\alpha}{\rho_i^2}\vI)\\
    =&\mathcal{N}(\bmuL_{i}|\frac{\alpha\hat{\net}_L(\bmuL_{i-1},\cdot)+\bmuL_{i-1}\rho_{i-1}}{\rho_{i}},\frac{\alpha}{\rho_i^2}\vI)\quad (\text{by property of isotropic Normal p.d.f})\\
    =&\update^L(\bmuL_{i}|\bmuL_{i-1},\hat{\net}_L(\bmuL_{i-1},\cdot),t_{i-1})\\
    =&p(\bmuL_{i}|\bmuL_{i-1},\parsnt{i-1}^A,\parsnt{i-1}^F)\\
    \end{align*}
    is equivariant. By \cref{lm:mrkv}, the marginal distribution $p(\bmuL_n)$ is O(3)-invariant.
\end{proof}
\fracinv*
\begin{proof}
We first prove that the Bayesian update is periodic translational equivariant. Based on \cref{eq:update_mi,eq:update_ci}, we can interpret the Bayesian update of $\{\dotm_{i-1},c_{i-1}\}$ observing $\doty$ with $\alpha$, as the \textit{vector addition} between $\dotm_{i-1}$ and $\doty$ with weight $c_{i-1}$ and $\alpha$. And the periodic translation $t$ for $x$ corresponds to the rotation of $\dotx$ with angle $t$:
 \begin{align*}
     &[\cos(x+t+2\pi k),\sin(x+t+2\pi k)]^T=[\cos(x+t),\sin(x+t)]^T\\
     =&[\cos x\cos t-\sin x\sin t,\sin x\cos t +\cos x\sin t]^T\\
     =&\begin{bmatrix}
\cos t & -\sin t \\
\sin t & \cos t
\end{bmatrix}\begin{bmatrix}
\cos x \\
\sin x 
\end{bmatrix}=\vec{R}_t\dot{\vx}
 \end{align*}
 where $\vec{R}_t$ is the 2-dimensional rotation matrix with angle $t$. Due to the rotational equivariance of 2D vector addition, we can infer that the Bayesian update function $h$ is periodic translational equivariant:
\begin{align}\label{eq:equiv_bu}
&h(\{w(m_{i-1}+t),c_{i-1}\},w(y+t),\alpha)=h(\{\vR_t\dotm_{i-1},c_{i-1}\},\vR_t\doty,\alpha)\nonumber\\
=&\{\frac{\alpha \vR_t\doty+c_{i-1}\vR_t\dotm_{i-1}}{||\alpha\vR_t\doty+c_{i-1}\vR_t\dotm_{i-1}||_2},||\alpha\vR_t\doty+c_{i-1}\vR_t\dotm_{i-1}||_2\}\nonumber\\
=&\{\frac{\alpha \vR_t\doty+c_{i-1}\vR_t\dotm_{i-1}}{||\alpha\doty+c_{i-1}\dotm_{i-1}||_2},||\alpha\doty+c_{i-1}\dotm_{i-1}||_2\}\nonumber\\
 =&\{\vR_t\dotm_i,c_i \}=\{w(m_{i}+t),c_{i}\}
\end{align}

    The prior is periodic translation invariant because $\bmF\sim U(\vec{0},\vec{1})$.\\
    We prove that the Bayesian update distribution $p_U(\bmF_{i}|\bmF_{i-1},\bcF_{i-1},g(\bmF_{i-1});\alpha)$ is periodic translation equivariant if $\net^F$ is periodic translation equivariant:
    \begin{align*}
    &p_U(w(\bmF_{i}+\vt)|w(\bmF_{i-1}+\vt),\bcF_{i-1},\net^F(w(\bmF_{i-1}+\vt));\alpha)\\
    =&\mathbb{E}_{vM(\bold{y}|\net^F(w(\bmF_{i-1}+\vt)),\alpha)}\delta(w(\bmF_{i}+\vt)-h(w(\bmF_{i-1}+\vt),\bcF_{i-1},\bold{y},\alpha))\\
    =&\mathbb{E}_{vM(\bold{y}|w(\net^F(\bmF_{i-1})+\vt),\alpha)}\delta(w(\bmF_{i}+\vt)-h(w(\bmF_{i-1}+\vt),\bcF_{i-1},\bold{y},\alpha))\text{(by equiv.} \net^F\text{)}\\
    =&\mathbb{E}_{vM(w(\bold{y}+\vt)|w(\net^F(\bmF_{i-1})+\vt),\alpha)}\delta(w(\bmF_{i}+\vt)-h(w(\bmF_{i-1}+\vt),\bcF_{i-1},w(\bold{y}+\vt),\alpha))\\
    =&\mathbb{E}_{vM(\y|\net^F(\bmF_{i-1})),\alpha)}\delta(w(\bmF_{i}+\vt)-h(w(\bmF_{i-1}+\vt),\bcF_{i-1},w(\bold{y}+\vt),\alpha))\text{(by \cref{eq:vm_period_equi})}\\
    =&\mathbb{E}_{vM(\y|\net^F(\bmF_{i-1})),\alpha)}\delta(w(\bmF_{i}+\vt)-w(h(\bmF_{i-1},\bcF_{i-1},\bold{y},\alpha)+\vt))\text{(by \cref{eq:equiv_bu})}\\
    =&\mathbb{E}_{vM(\y|\net^F(\bmF_{i-1})),\alpha)}\delta(\bmF_{i}-h(\bmF_{i-1},\bcF_{i-1},\bold{y},\alpha))  \text{(by equivariance of }\delta \text{ function)}\\
    =&p_U(\bmF_{i}|\bmF_{i-1},\bcF_{i-1},\net^F(\bmF_{i-1});\alpha)\\
    \end{align*}    
\end{proof}

Next, we prove the following proposition:
\cirflowequiv*
\begin{proof}
Combining \cref{eq:update_mi,eq:update_ci},
\begin{align*}
\dotm_{i}\vec{c}_i&=\alpha_i \doty_i+\vec{c}_{i-1}\dotm_{i-1}\\
&=\alpha_i \doty_i + \alpha_{i-1} \doty_{i-1}+\vec{c}_{i-2}\dotm_{i-2}\\
&=\alpha_i \doty_i + \dots + \alpha_{1} \doty_{1}+\vec{c}_0 \dotm_{0}\\
&=\sum_{j=1}^i \alpha_i \doty_i=[\sum_{j=1}^i \alpha_i\cos \y_i,\sum_{j=1}^i \alpha_i\sin \y_i]^T
\end{align*}
Taking the 2-norm to each side,
\begin{align*}
    ||\dotm_{i}\vec{c}_i||_2&=||[\sum_{j=1}^i \alpha_i\cos \y_i,\sum_{j=1}^i \alpha_i\sin \y_i]^T||_2\\
    \vec{c}_i&=||[\sum_{j=1}^i \alpha_i \cos \y_i,\sum_{j=1}^i \alpha_i \sin \y_i]^T||_2
\end{align*}
The vector direction of $\dotm_{i}$ is irrelevant to the scaler $\vec{c}_i$. Therefore,
\begin{equation}
    \dotm_{i}=\text{atan2}(\sum_{j=1}^i \alpha_i \cos \y_i,\sum_{j=1}^i \alpha_i \sin \y_i)
\end{equation}

Hence,
\begin{align*}
&p_F(\vec{m}_i|\x;\alpha_1,\alpha_2,\dots,\alpha_i)\\=&\E_{\update(\parsnt{1} \mid \parsnt{0}, \x ; \alphat{1})}\dots\E_{\update(\parsn_{i-1} \mid \parsnt{i-2}, \x; \alphat{i-1})} \update(\vm_{i} | \parsnt{i-1},\x;\alphat{i} ) \\
=& \E_{vM(\y_1|\x,\alpha_1)}\dots\E_{vM(\y_i|\x,\alpha_i)} \delta(\vec{m}_{i}-\text{atan2}(\sum_{j=1}^i \alpha_i \cos \y_i,\sum_{j=1}^i \alpha_i \sin \y_i))\\
&p_F(\vec{c}_i|\x;\alpha_1,\alpha_2,\dots,\alpha_i)\\=&\E_{\update(\parsnt{1} \mid \parsnt{0}, \x ; \alphat{1})}\dots\E_{\update(\parsn_{i-1} \mid \parsnt{i-2}, \x; \alphat{i-1})} \update(\vm_{i} | \parsnt{i-1},\x;\alphat{i} ) \\
=& \E_{vM(\y_1|\x,\alpha_1)}\dots\E_{vM(\y_i|\x,\alpha_i)} \delta(\vec{c}_{i}-||[\sum_{j=1}^i \alpha_i\cos \y_i,\sum_{j=1}^i \alpha_i\sin \y_i]^T||_2)
\end{align*}
\end{proof}


\section{Implementation Details}\label{appd:imple_details}

\textbf{Training and Sampling Procedure} We provide the training and sampling procedure in \cref{alg:train} and in \cref{alg:sampling}. 

\textbf{Network Architecture} We use CSPNet proposed by \citet{jiao2023crystal} with minor modifications: $(1)$ We add a residual connection from the input to the output of fractional coordinates ensuring the equivariance of the network:
\begin{equation}
\net^F(\bthetaA_i,\bthetaF_i,\bthetaL_i,t_i) = w(\varphi_F(\vh_i^{(S)})+\bthetaF_i),
\end{equation}
By the periodic translational invariance of $\varphi_F(\vh_i^{(S)})$ proved by \citet{jiao2023crystal}, the equivariance of $\predF{i}$ can be easily checked:
\begin{align*}
   \net^F(\bthetaA_i,w(\bthetaF_i+\vt),\bthetaL_i,t_i)= &w(\varphi_F(\vh_i^{(S)})+w(\bthetaF_i+\vt)) \\
   =& w(w(\varphi_F(\vh_i^{(S)})+\bthetaF_i)+\vt) \\
   =& w(\net^F(\bthetaA_i,\bthetaF_i,\bthetaL_i,t_i)+\vt)
\end{align*}

$(2)$We alter the frequency of the Fourier transformation features to model in the interval $[-\pi,\pi)$ with length $2\pi$. $(3)$The concentration parameter of each fractional coordinate $\bcF$ is taken logarithm, normalized and concatenated to time embedding. The network hyper-parameters follows the setting of \citet{jiao2023crystal} including the number of hidden states and layers. 

\textbf{Hyper-parameters} For the network, the CSPNet has 6 layers, 512 hidden states, 128 frequencies for the Fourier feature for each task and dataset following \citep{jiao2023crystal}. For BFN hyper-parameters, we set $\sigma_1^2=0.001$ for continuous variable generation, $\beta_1=1000$ for circular variables generation across all datasets and tasks. For discrete variables, we set $\beta_1=0.4$ for the MP-20 dataset and $\beta_1=3.0$ for the Perov-5 dataset. The number of steps is searched in $\{50,100,500,1000,2000\}$. For optimizations, we apply an AdamW optimizer with an initial learning rate $1\times10^-3$ and a plateau scheduler with a decaying factor of 0.6, a patience of 100 epochs, and a minimal learning rate $1\times10^{-4}$. The weight of every loss is $5\times10^{-2}$. The network is trained for 4000, 5000, 1500, and 1000 epochs for Perov-5, Carbon-24, MP-20, and MPTS-52 respectively.

\textbf{Computational Resources} All training experiments are conducted on a server with 8 $\times$ NVidia RTX 3090 GPU, 64 $\times$ Intel Xeon Platinum 8362 CPU and 256GB memory. Each training task requires one GPU. We also report the required GPU hour across methods to converge in our experimental environment in \cref{tab:train_gpu_hour}.

\begin{table*}[h!]
  \centering
  \caption{Comparison of GPU hours required for training across different methods.}
  \resizebox{0.6\linewidth}{!}{
  \small
    \setlength{\tabcolsep}{2.5pt}
    \begin{tabular}{lccccccc}
    \toprule
    \textbf{GPU Hour} & Perov-5  &  MP-20 & MPTS-52 \\
    \midrule
    DiffCSP~\cite{jiao2023crystal}    & 8.59 & 92.22 & 10.42 \\
    FlowMM~\cite{flowmm}    & 16.36 & 106.37 & 16.49 \\
    \modelname    & 10.19  & 85.71 & 12.31 \\
    \bottomrule
    \end{tabular}
    }
\vskip -0.15in
 \label{tab:train_gpu_hour}%
\end{table*}%

\section{More Results}
\textbf{Visualizations} Here we give visualizations of the ab-initio generated structures from CrysBFN and DiffCSP in \cref{fig:vis_apd}. We also provide a gif animation of the generation process in our code repository \url{https://github.com/wu-han-lin/CrysBFN}.
\begin{figure*}[h!]
    \centering    \includegraphics[width=\textwidth]{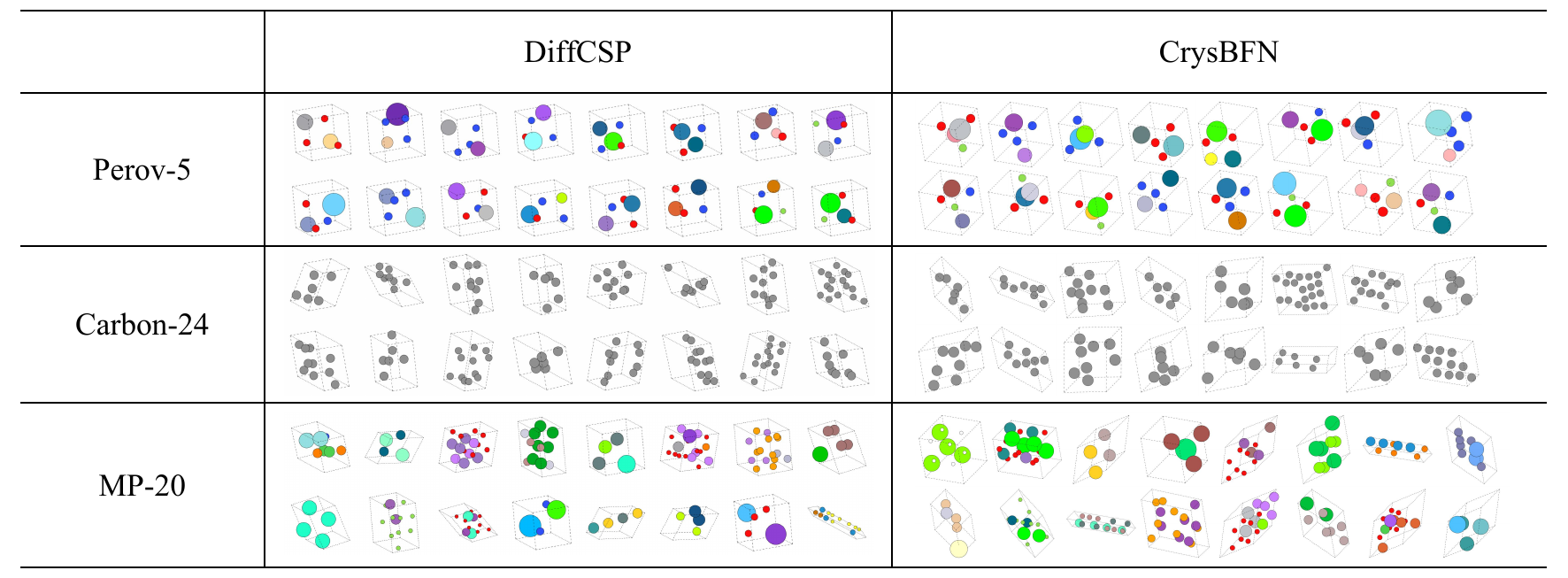}
    \caption{Visualizations comparison of the ab-initio generated structures from CrysBFN and DiffCSP.}
    \label{fig:vis_apd}
\end{figure*}

\textbf{Error Bars} We report the error bar for the crystal structure prediction task following \citet{jiao2023crystal} in \cref{tab:std} running three experiments with different random seeds. The results are similar to \cref{tab:oto}.
\begin{table*}[h]
  \centering
  \caption{Results on Perov-5 and MP-20 with error bars. }
  \small
    \setlength{\tabcolsep}{2.5pt}
    \begin{tabular}{p{4cm}ccccc}
    \toprule
     &  \multicolumn{2}{c}{Perov-5 } & \multicolumn{2}{c}{MP-20} \\
             & Match rate (\%)$\uparrow$ & RMSE$\downarrow$ & Match rate (\%)$\uparrow$ & RMSE$\downarrow$  \\
    \midrule
    {CDVAE~\citep{xie2021crystal}}& 45.31$\pm$0.49  & 0.1123$\pm$0.0026 &  33.93$\pm$0.15  & 0.1069$\pm$0.0018  \\
    \midrule
    DiffCSP~\citep{jiao2023crystal}  & 52.35$\pm$0.26  & 0.0778$\pm$0.0030 & 51.89$\pm$0.30 &	0.0611$\pm$0.0015 \\
    \midrule
    \modelname  & \textbf{54.58$\pm$0.13}  & \textbf{0.0691$\pm$0.0011} & \textbf{64.33$\pm$0.24} &	\textbf{0.0445$\pm$0.0010} \\
    \bottomrule
    \end{tabular}%
  \label{tab:std}%
\end{table*}%

\textbf{Uniqueness, Novelty, and Stability} Here we compare the uniqueness, novelty, and stability of ab initio generated samples across methods on MP-20. Using StructureMatcher in pymatgen with default parameters, a generated crystal is considered: 1) unique if it does not match any other generated samples; and 2) novel if it does not match any crystals in the training set, following prior practices \citep{zeni2023mattergen,flowmm}. The stability evaluation procedure follows the approach in \citet{crysllm}. Finally, a sample is considered stable, unique, and novel (S.U.N.) if it satisfies all three conditions. The results are reported in \cref{tab:sun}.
\begin{table*}[h]
  \centering
  \caption{Uniqueness, novelty, and stability experimental results comparison on ab initio generation task on MP-20 dataset.}
  \small
    \setlength{\tabcolsep}{2.5pt}
    \begin{tabular}{p{4cm}cccccc}
    \toprule
    \textbf{Method}& \textbf{Unique} / \% & \textbf{Novel} / \% & \textbf{Metastable} / \% & \textbf{Stable} / \% & \textbf{S.U.N. Rate} / \% \\
    \midrule
    DiffCSP~\citep{jiao2023crystal}    & \textbf{96.11} & 90.95 & 37.91 & 12.16 & 9.44 \\
    \midrule
    FlowMM~\citep{flowmm}    & 94.79 &	91.63 & 32.77 & 9.23 & 8.31 \\
    \midrule
    \modelname  & 95.29 & \textbf{92.37} & \textbf{45.91} & \textbf{15.82} & \textbf{12.16}\\
    \bottomrule
    \end{tabular}%
  \label{tab:sun}%
\end{table*}%

\begin{table*}[t!]
 %
  \centering
  \caption{Ablation study results of entropy-conditioning mechanism across datasets.}
  \resizebox{0.90\linewidth}{!}{
  \small
    \setlength{\tabcolsep}{2.5pt}
    \begin{tabular}{lccccccc}
    \toprule
    \multirow{2}[2]{*}{} & \multicolumn{2}{c}{Perov-5 } &  \multicolumn{2}{c}{MP-20} & \multicolumn{2}{c}{MPTS-52} \\
          &  Match rate$\uparrow$ & RMSE$\downarrow$   & Match rate$\uparrow$ & RMSE$\downarrow$ & Match rate$\uparrow$ & RMSE$\downarrow$ \\
    \midrule
    w/o entropy conditioning    & 51.33 & 0.0753 & 52.16 & 0.0631 & 13.41 & 0.1547 \\
    \modelname    & \textbf{54.69}  & \textbf{0.0636} & \textbf{64.35} &	\textbf{0.0433} & \textbf{20.52} & \textbf{0.1038} \\
    \bottomrule
    \end{tabular}
    }
 \label{tab:entropy_condition_datasets}%
\end{table*}%

\section{Sampling Efficiency Comparison to ODE Samplers}
\begin{figure}[h!] 
    \centering
    \includegraphics[width=0.5\textwidth]{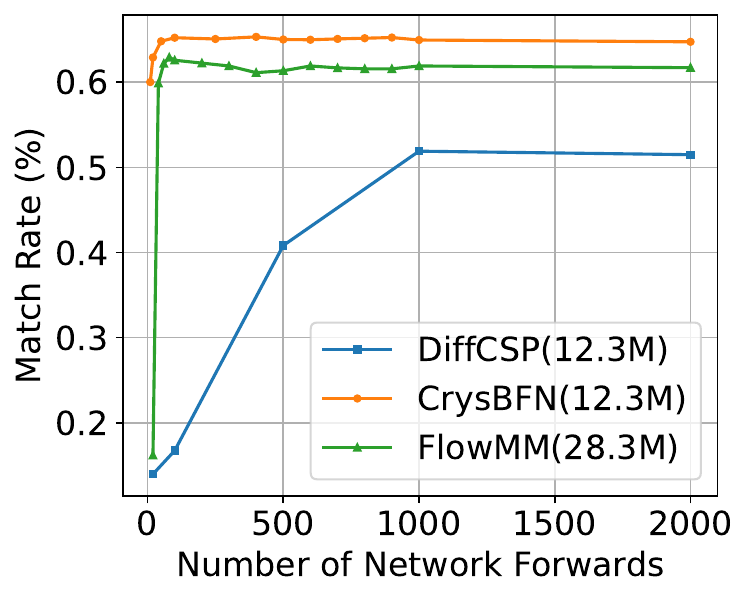} 
    \caption{Experimental results on MP-20 with different Number of Function Evaluations (NFE) \emph{i.e.} number of network forward passes including FlowMM. Note that FlowMM has a larger parameter size resulting in a less fair comparison.} 
    \label{fig:ode_compare}
\vskip -0.2in
\end{figure}
We present the comparison among them in \cref{fig:ode_compare} and find that FlowMM fails in extremely small NFE (20 steps) settings with only 16.18\% match rate, while CrysBFN enjoys 60.02\% match rate with 10 steps and consistently achieves the best sampling quality. 

\section{Detailed Discussion of Related Works}
\label{appd:detailed_related_works}
Discovering new functional materials has been a long-standing scientific problem. Recently, data-driven approaches have been seen as a promising solution to address this challenge ~\cite{peng2022human}.

\textbf{Two-stage crystal generation methods based on implicit crystal representations} One line of approaches indirectly generates crystals in the implicit representation space. Prior practices includes transforming crystals into human-designed fingerprint FTCP~\citep{REN2021}, 3D voxels images~\citep{hoffmann2019data}, 2D images~\citep{noh2019inverse}, 3D electron-density maps~\citep{court20203}, video-like representation~\citep{yang2023scalable}, embedded atom density \citep{zhang2019embedded} in StructRepDiff~\citep{sinha2024representation}. However, their generation quality is hampered by the encoding and reconstruction processes, which may not be fully reversible or fail to respect physical symmetries such as rotational and translational invariance. For example, 3D voxel grids ~\citep{hoffmann2019data} and 3D density maps ~\citep{court20203} are invariant to periodic transformations but not to \(E(3)\) transformations~\citep{zhang2023artificial}, and the video-like representation~\citep{yang2023scalable} is not invariant to permutation, rotation, translation, and periodic transformations. 

\textbf{Direct crystal generation methods} Direct material generation in sample space could bypass the above reversibility problem. Prior works \citep{nouira2018crystalgan,kim2020generative} employ Generative Adversarial Networks \citep{goodfellow2020generative} to generate crystal structures, while their methods fail to respect crystal geometric invariance. Inspired by the success of Diffusion models on images~\citep{ho2020denoising,song2020score,song2019generative}, the multi-step generation paradigm has been introduced into generative modeling of atom systems including molecular conformations~\citep{xu2023geometric}. The geometric invariance of the generation path can be guaranteed by designing a Markov chain with an invariant prior and equivariant transitions \citep{xu2021geodiff}. CDVAE~\citep{xie2021crystal}, its CSP adaptation Cond-CDVAE~\citep{luo2024deep} and SyMat~\citep{luo2024towards}, generate crystalline materials levearging $E(3)$-equivariant graph neural network models~\citep{klipfel2023unified,gasteiger2021gemnet} on 3D multi-edge graphs. Utilizing VAE models, they generate lattice parameters, randomly initialize atom coordinates, and iteratively refine these coordinates using score-matching models~\cite{song2019generative}. With the fractional coordinate system, DiffCSP \citep{jiao2023crystal} firstly introduces the periodic E(3) equivariance of crystals and designs an equivariant diffusion crystal generation model based on periodic diffusion \citep{jing2022torsional}. Subsequently, FlowMM \citep{flowmm} recently introduced Riemannian Flow Matching \citep{riemannianfm} for the task of crystal generation, offering improved sampling efficiency, albeit at the expense of quality.

We argue that such struggles of balancing between sampling quality and efficiency stems from the lack of proper guidance on each transition from noise prior to data distribution especially for crystals, where thermodynamically stable materials represent
only a small fraction in the search space~\citep{flowmm}. For example, early generation states $x_{t-1}$ with low confidence should be preserved less than the later ones to get the next state $x_{t}$. From the perspective of Bayesian updates, Bayesian Flow Networks~\citep{bfn} provides a framework to accurately update each $m_{t-1}$ according to its confidence parameter $\alpha_i$, the effectiveness of which has been proved in \citet{song2023unified}. However, periodicity is not considered in \citet{bfn} and incorporating it into BFN is non-trivial without desirable distributions with mathematical properties like Gaussian. To address above issues, in this paper, we build a Bayesian flow almost from scratch, identifying and tackling the non-additive accuracy via introducing a novel entropy conditioning mechanism, theoretical reformulations of BFN, a fast sampling algorithm, \emph{etc}. We demonstrate the effectiveness of the guidance of entropy in \cref{tab:exp_abl} and \cref{tab:entropy_condition_datasets} and its higher sampling effiency and quality in \cref{fig:eff_compare,fig:ode_compare}.

 Additionally, recently various techniques have also been introduced to boost the performance tailored for crystal generation considering crystal's inductive bias, including \citet{jiao2024space,cao2024space,ai4science2023crystal} which incorporate the space group constraint into the generation process. Recently, EquiCSP~\citep{equicsp} proposed to utilize a periodic CoM-free noising method and introduce lattice permutation invariance loss. Those techniques are orthogonal to the proposed method in this paper.


 \begin{algorithm}[h]
\small
\caption{Training Procedure}\label{alg:train}
\begin{algorithmic}[1]
\STATE \textbf{Require:} number of steps $n\in\R$, $\sigma_1\in\R^{+},\beta_1\in\R^{+},\alpha_1^F,\dots,\alpha_n^F\in\R^{+}$. 

\STATE \textbf{Input:}  atom types $\vec{A}$, fractional coordinates $\vF$, lattice parameter $\vL$, length of vocabulary $K$
\STATE Sample $i\sim U\{1,n\}$, $t\leftarrow \frac{(i-1)}{n}$
\STATE \textcolor{gray}{\# sampling from Bayesian flow distribution of lattice}
\STATE $\gamma^L \leftarrow 1-\sigma_1^{2t}$
\STATE $\vmu_L\sim\mathcal{N}(\gamma^L \vL,\gamma^L(1-\gamma^L)\vI)$
\STATE \textcolor{gray}{\# sampling from Bayesian flow distribution of atom type}
\STATE $\beta^A \leftarrow \beta_1 t^2$
\STATE $\y_A' \sim \N{\beta\left(K\oh{\x}{KD}-
\1{KD}\right)}{\beta K\vI}$
\STATE $\parsn^A \gets \text{softmax}(\y_A')$
\STATE \textcolor{gray}{\# sampling from Bayesian flow distribution of fractional coordinates}
\STATE $\y_1,\y_2,\dots,\y_i \gets vM(\vF,\alpha_1^F),\dots vM(\vF,\alpha_i^F)$
\STATE $\vec{m}_{i}\gets\text{atan2}(\sum_{j=1}^i \alpha_j \cos \y_j,\sum_{j=1}^i \alpha_j \sin \y_j)$
\STATE \textcolor{gray}{\# calculate the accumulated accuracy \textit{i.e.} entropy}
\STATE $\vec{c}_{i}\gets ||[\sum_{j=1}^i \alpha_j\cos \y_j,\sum_{j=1}^i \alpha_j\sin \y_j]^T||_2$
\STATE \textcolor{gray}{\# use network to do inter-dependency modeling across dimensions conditioning on entropy}
\STATE $\bthetaM\gets (\vmu_L,\parsn^A,\vec{m}_{i},\vec{c}_{i})$
\STATE $\predL{i-1},\predF{i-1},\predF{i-1} \gets \net(\bthetaM,t)$ 
\STATE \textcolor{gray}{\# calculate the losses of all modalities }
\STATE $\alpha^A_i \leftarrow \beta_1^A\left(\frac{2i -1}{n^2}\right)$
\STATE $\y_A \sim \N{\alpha\left(K\oh{\x}{KD}-
\1{KD}\right)}{\alpha K\vI}$
\STATE $\calL_A \gets n\ln \N{\y_A \mid \alphat{i}^A\left(K \oh{\vA}{K\times N} - \vec{1}\right)}{\alphat{i}^A K \vec{I}}$
\STATE\qquad\qquad$-\sum_{d=1}^N \ln \left(\sum_{k=1}^K \out^{(d)}(k \mid \parsn^A; t_{i-1}) \N{\ydd{d}_A \mid \alphat{i}^A\left(K\oh{k}{K}- \vec{1}\right)}{\alphat{i}^A K \vec{I}}\right)$
\STATE $\calL_F \gets n~\alpha_i^F\frac{I_1(\alpha_i^F)}{I_0(\alpha_i^F)}(1-\cos(\vF-\predF{i-1}))$
\STATE $\mathcal{L}_{L} = \frac{n}{2}\left(1-\sigma_1^{2/n}\right)\frac{\left\|\vL -\predL{i-1}\right\|^2}{\sigma_1^{2i/n}}$
\STATE Minimize $\calL_A+\calL_F+\calL_L$
\end{algorithmic}
\end{algorithm}

\begin{algorithm}[h]
\small
\caption{Sampling Procedure}\label{alg:sampling}
\begin{algorithmic}
\STATE \textbf{Require:} number of steps $n\in\R$, length of vocabulary $K$, $\sigma_1\in\R^{+},\beta_1\in\R^{+},\alpha_1^F,\dots,\alpha_n^F\in\R^{+}$
\STATE \textcolor{gray}{\# initialize the prior parameters }
\STATE $\vmu_0$,~$\rho_0\gets \vec{1}$
\STATE $\vec{\theta}_0\gets\frac{1}{K}\vec{1}_{1\times N}$
\STATE $\vm_0\gets U(\vec{0},\vec{1}),\vc_0\gets \vec{0}$
\FOR{$i \gets 1,\cdots, n$}
    \STATE{$t\gets\frac{i-1}{n}$}
    \STATE \textcolor{gray}{\# use network to do inter-dependency modeling across dimensions of all modalities }
    \STATE{$\bthetaM\gets (\vmu_{i-1},\parsn_{i-1},\vec{m}_{i-1},\vec{c}_{i-1})$}
    \STATE{$\predL{i-1},\predF{i-1},\predF{i-1} \gets \net(\bthetaM,t)$}
    \IF{$i<n$}
    \STATE \textcolor{gray}{\# do Bayesian update for lattice parameter }    
    \STATE{$\alpha_i^L\gets \sigma_1^{-2i/n}(1-\sigma_1^{2/n})$}
    \STATE{$\y^L\sim \mathcal{N}(\predL{i-1},\frac{1}{\alpha_i^L}\vI)$}
    \STATE{$\vmu_i\gets \frac{\rho_{i-1}\vmu_{i-1}+\alpha_i^L\y^L}{\rho_{i-1}+\alpha_i^L}$}
    \STATE{$\rho_i\gets \rho_{i-1}+\alpha_i^L$}
    \STATE \textcolor{gray}{\# do Bayesian update for fractional coordinates }   
    \STATE{$\y^F\sim vM(\predF{i-1},\alpha_i^F)$}
    \STATE{$\vm_i\gets \text{atan2}(\alpha_i^F\sin \y+\vc_{i-1}\sin \vm_{i-1}, {\alpha_i^F\cos \y+\vc_{i-1}\cos \vm_{i-1}})$}
    \STATE{$\vc_i\gets ||[\alpha_i^F\sin \y+\vc_{i-1}\sin \vm_{i-1}, {\alpha_i^F\cos \y+\vc_{i-1}\cos \vm_{i-1}}]^T||_2$}
    \STATE \textcolor{gray}{\# do Bayesian update for atom types }   
    \STATE $\alpha_i^A \leftarrow \beta_1\left(\frac{2i -1}{n^2}\right)$
    \STATE $\y^A \sim \N{\alpha_i^A\left(K\oh{\k}{KD}-
    \vec{1}\right)}{\alpha_i^A K\vI}$
    \STATE $\parsn' \gets e^{\y^A} \parsn_{i-1}$
    \STATE $\parsn_i \gets \frac{\parsn'}{\sum_k \parsn'_k}$
    \ENDIF
\ENDFOR
\STATE $\hat{\vec{A}}\sim \predA{n-1}$ \textcolor{gray}{\# sample from the final probability prediction}   
\STATE \textbf{Return} $\hat{\vec{A}},\predF{n-1},\predL{n-1}$
\end{algorithmic}
\end{algorithm}